\documentclass[12pt]{elsarticle}

\RequirePackage{etex}

\usepackage{lineno}

\usepackage[utf8]{inputenc}
\usepackage{pgfplots}
\DeclareUnicodeCharacter{2212}{−}
\usepgfplotslibrary{groupplots,dateplot}
\usetikzlibrary{patterns,shapes.arrows}
\pgfplotsset{compat=newest}

\usepackage[textwidth=16.31cm]{geometry}
\usepackage{stmaryrd}
\usepackage{bbm}	
\usepackage{comment} 
\usepackage{amsthm,amssymb,mathtools}
\usepackage{thmtools,thm-restate}
\usepackage{tikz}
\usepackage{tikz-network}
\usetikzlibrary{arrows,automata}
\usepackage{multirow}
\usepackage{natbib}
\usepackage{cancel}

\newtheorem{lemma}{Lemma}
\newtheorem{theorem}{Theorem}
\newtheorem{proposition}{Proposition}
\newtheorem{definition}{Definition}
\newtheorem{example}{Example}
\DeclareMathOperator{\PP}{\mathcal{P}}
\DeclareMathOperator{\RankSet}{\mathcal{R}}
\DeclareMathOperator{\rk}{\mathtt{rk}}
\DeclareMathOperator{\MG}{\mathcal{G}}
\DeclareMathOperator{\tp}{\mathtt{top}}

\DeclareMathOperator{\KendallTau}{\delta_{\mathtt{KT}}}

\newcommand{\TopSortSet}[1]{\mathcal{O}(#1)}

\newcommand{\PPr}[1]{\ensuremath{\mathcal{P}_{#1}}}
\newcommand{\SubsetSet}[1]{\ensuremath{\Delta^{#1}}}
\newcommand{\EKendallTau}[1]{\ensuremath{\delta_{\mathtt{KT}}^{#1}}}

\newcommand{\WKendallTau}[1]{\ensuremath{K_{#1}}}
\newcommand{\EMG}[1]{\ensuremath{\mathcal{G}^{#1}}}
\newcommand{\revised}[1]{\textcolor{black}{#1}}
\newcommand{\revtwo}[1]{\textcolor{black}{#1}}

\DeclareMathOperator{\disagree}{\revtwo{\texttt{disagree}}}
\DeclareMathOperator{\bel}{\revtwo{\texttt{Below}}}
\DeclareMathOperator{\Spearman}{\delta_{\mathtt{S}}}
\newcommand{\ESpearman}[1]{\ensuremath{\delta_{\mathtt{S}}^{#1}}}

\newcommand{\HG}[1]{\textcolor{black}{#1}}

\journal{Theoretical Computer Science}

\makeatletter
\def\ps@pprintTitle{%
 \let\@oddhead\@empty
 \let\@evenhead\@empty
 \def\@oddfoot{}%
 \let\@evenfoot\@oddfoot}
\makeatother
\begin{document}

\begin{frontmatter}

\title{Beyond Pairwise Comparisons in Social Choice:\\ A Setwise Kemeny Aggregation Problem\tnoteref{longVersion}}
\tnotetext[longVersion]{This is a long version of a work presented at the 34th {AAAI} Conference on Artificial Intelligence ({AAAI} 2020) \cite{GilbertPS20}.}

\author[hg]{Hugo Gilbert}
\ead{hugo.gilbert@dauphine.psl.eu}
\author[tp]{Tom Portoleau}
\ead{tom.portoleau@laas.fr}
\author[os]{Olivier Spanjaard\corref{cor1}}
\ead{olivier.spanjaard@lip6.fr}
\cortext[cor1]{Corresponding author.}
\address[hg]{Universit\'e Paris-Dauphine, Universit\'e PSL, CNRS, LAMSADE, 75016 Paris, France}
\address[tp]{LAAS-CNRS, IRIT-CNRS, Universit\'e de Toulouse, Toulouse, France}
\address[os]{Sorbonne Universit\'e, CNRS, Laboratoire d'Informatique de Paris 6, LIP6, Paris, France}


\begin{abstract}
In this paper, we advocate the use of setwise contests for aggregating a set of input rankings into an output ranking. We propose a generalization of the Kemeny rule where one minimizes the number of $k$-wise disagreements instead of pairwise disagreements (one counts 1 disagreement each time the top choice in a subset of alternatives of cardinality at most $k$ differs between an input ranking and the output ranking). After an algorithmic study of this $k$-wise Kemeny aggregation problem, we introduce a $k$-wise counterpart of the majority graph. \HG{This graph} reveals useful to divide the aggregation problem into several sub-problems, \revised{which enables to speed up the exact computation of a consensus ranking}. \HG{By introducing a $k$-wise counterpart of the Spearman distance, we also provide a 2-approximation algorithm for the $k$-wise Kemeny aggregation problem.} We conclude with numerical tests. 
\end{abstract}


\end{frontmatter}


\section{Introduction\label{sec:Intro}}
Rank aggregation aims at producing a single ranking from a collection of rankings of a fixed set of alternatives. In social choice theory \revised{(see, for example, the book by \citeauthor{moulin1991axioms} \cite{moulin1991axioms})}, where the alternatives are candidates to an election and each ranking represents the preferences of a voter, aggregation rules are called \emph{Social Welfare Functions} (SWFs). Apart from social choice, rank aggregation has proved useful in many applications, including preference learning \citep{chengHullermeier2009,ClemenconKS18}, collaborative filtering \citep{wang2014vsrank}, genetic map creation \citep{jackson2008consensus}, similarity search in database systems
\citep{fagin2003efficient} and design of web search engines \citep{altman2008axiomatic,dwork2001rank}. In the following, we use interchangeably the terms ``input rankings'' and ``preferences'', ``output ranking'' and ``consensus ranking'', as well as ``alternatives'' and ``candidates''. 

The well-known Arrow's impossibility theorem states that there exists no aggregation rule satisfying a small set of desirable properties \citep{arrow1950difficulty}. In the absense of an ``ideal'' rule, various aggregation rules have been proposed and studied. Following Fishburn's classification \cite{fishburn1977condorcet}, we can distinguish between the SWFs for which the output ranking can be computed from the \emph{majority graph} alone, those for which the output ranking can be computed from the \emph{weighted majority graph} alone, and all other SWFs. The majority graph is obtained from the input rankings by defining one vertex per alternative $c$ and by adding an edge from $c$ to $c'$ if $c$ is preferred to $c'$ in a strict majority of input rankings. In the weighted majority graph, each edge is weighted by the majority margin.  
The many SWFs that rely on these graphs alone take therefore only pairwise comparisons into account to determine an output ranking. \revised{Note that Fishburn's classification actually applies to social \emph{choice} functions, which prescribe a subset of winning alternatives from a collection of rankings, but the extension to SWFs is straightforward. For a compendium of SWFs that rely solely on (weighted) majority graphs, we refer the reader to the book chapter by \citeauthor{zwicker2016introduction} \cite{zwicker2016introduction}. }

The importance of this class of SWFs can be explained by their connection with the \emph{Condorcet consistency} property, stating: if there is a Condorcet winner (i.e., an alternative with outgoing edges to every other ones in the majority graph), then it should be ranked first in the output ranking. 
Nevertheless, as shown by \citeauthor{baldiga2010choice} \cite{baldiga2010choice}, the lack of Condorcet consistency is not necessarily a bad thing, because this property may come into contradiction with the objective of maximizing voters' agreement with the output ranking. The following example illustrates this point.

\begin{example}[Baldiga and Green \cite{baldiga2010choice}] \label{ex:Baldiga}
Consider an election with 100 voters and 3 candidates $c_1,c_2,c_3$, where 49 voters have preferences $c_1 \succ c_2 \succ c_3$, 48 have preferences $c_3 \succ c_2 \succ c_1$ and 3 have preferences $c_2 \succ c_3 \succ c_1$. Candidate $c_2$ is the Condorcet winner, but is the top choice of only 3 voters. In contrast, candidate $c_1$ is in slight minority against $c_2$ and $c_3$, but $c_1$ is the top choice of 49 voters. This massive gain in agreement may justify to put $c_1$ instead of $c_2$ in first position of the output ranking.
\end{example}

Following Baldiga and Green \cite{baldiga2010choice},
we propose to handle this tension between the pairwise comparisons (leading to ranking $c_2$ first) and the plurality choice (leading to ranking $c_1$ first) by using SWFs that take into account not only pairwise comparisons but \emph{setwise} contests. More precisely, given input rankings on a set $C$ of candidates and $k\!\in\!\{2,\ldots,|C|\}$, the idea is to consider the plurality score of each candidate $c$ for each subset $S\!\subseteq\!C$ such that $2 \!\le\!|S|\!\le\!k$, where the plurality score of $c$ for $S$ is the number of voters for which $c$ is the top choice in $S$. The results of setwise contests for the preferences of Example~\ref{ex:Baldiga} are given in Table~\ref{tab:setwise_contests} for $k\!=\!3$. Note that the three top rows obviously encode the same information as the weighted majority graph while the bottom row makes it possible to detect the tension between the pairwise comparisons and the plurality choice. 

\begin{table}
    \centering
\caption{\label{tab:setwise_contests} Results of setwise contests in Example~\ref{ex:Baldiga}. \HG{
The cell at the intersection of the row corresponding to a set $S$ and the column corresponding to a candidate $c$ displays the number of voters who rank $c$ first in $S$.}}
    \begin{tabular}{cccc}
\hline
set & $c_1$ & $c_2$ & $c_3$  \\
\hline
$\{c_1,c_2\}$ & 49 & 51 & -- \\
$\{c_1,c_3\}$ & 49 & -- & 51 \\
$\{c_2,c_3\}$ & -- & 52 & 48 \\
$\{c_1,c_2,c_3\}$ & 49 & 3 & 48 \\
\hline
\end{tabular}
\end{table}

One can then define a new class of SWFs, those that rely on the results of setwise contests alone to determine an output ranking. The many works that have been carried out regarding voting rules based on the (weighted) majority graph can be revisited in this broader setting. This line of research has already been investigated by Lu and Boutilier \cite{lu2010unavailable} and Baldiga and Green \cite{baldiga2010choice}. However, note that both of these works consider a setting where candidates may become unavailable after voters express their preferences. 
We do not make this assumption. We indeed believe that this new class of SWFs makes sense in the standard setting where the set of candidates is known and deterministic, as it amounts to generate an output ranking by examining the choices that are made by the voters on subsets of candidates of various sizes (while usually only pairwise choices are considered).

A natural SWF in this class consists in determining an output ranking that minimizes the number of disagreements with the results of setwise contests for sets of cardinality at most $k$. This is a $k$-wise generalization of the Kemeny rule, obtained as a special case for $k\!=\!2$. We recall that the Kemeny rule consists in producing a ranking that minimizes the number of \emph{pairwise} disagreements \citep{kemeny1959mathematics}.

\begin{example}
\label{ex:kwise}
Let us come back to Example~\ref{ex:Baldiga} and assume that we use the 3-wise Kemeny rule.
Consider the output ranking $r = c_1\!\succ\!c_2\!\succ\!c_3$. For set $S\!=\!\{c_1,c_2\}$, the number of disagreements with the results of setwise contests is $51$ because $c_2$ is the top choice in $S$ for $51$ voters (see Table~\ref{ex:Baldiga}) while it is $c_1$ for $r$. Similarly, the number of disagreements induced by $\{c_1,c_3\}$, $\{c_2,c_3\}$ and $\{c_1,c_2,c_3\}$ are respectively $51$, $48$ and $3\!+\!48$. The total number of disagreements is thus $51\!+\!51\!+\!48\!+\!3\!+\!48\!=\!201$. This is actually the minimum number of disagreements that can be achieved for these input rankings, which makes $r$ the $k$-wise Kemeny ranking.
\end{example}

The purpose of this paper is to study the $k$-wise Kemeny aggregation problem. Section~\ref{sec:Preliminaries} formally defines the problem and reports on related work. 
Section~\ref{sec:CompAndExact} is devoted to some axiomatic considerations of the corresponding voting rule, and to an algorithmic study of the problem. 
\HG{More precisely, we show that the decision variant of the $k$-wise Kemeny aggregation problem is NP-complete for any constant $k \ge 2$ and we provide an efficient fixed-parameter algorithm for parameter $m$ which relies on dynamic programming.} 
We then investigate a $k$-wise variant of the majority graph in Section~\ref{sec:Graph}. We prove that determining this graph is easy for $k\!=\!3$ but becomes NP-hard for $k\!>\!3$, and we show how to use it in a preprocessing step to speed up the computation of the output ranking. 
In \HG{Section~\ref{sec:approx}, we propose a 2-approximation algorithm for the $k$-wise Kemeny aggregation problem, by introducing a $k$-wise variant of the Spearman distance.}   
Numerical tests are presented in Section~\ref{sec:Numerical} \HG{to assess both the efficiency of our exact methods and the accuracy of our approximation algorithm}. 

\section{Preliminaries\label{sec:Preliminaries}}
Adopting the terminology of social choice theory, we consider an election with a set \(V\) of \(n\) voters and a set \(C\) of \(m\) candidates. 
Each voter \(v\) has a complete and transitive preference order \(r_v\) over candidates (also called ranking). 
The collection of these rankings defines a preference profile \(\PP\). 

\subsection{Notations and Definitions}

Let us introduce some notations related to rankings. We denote by \(\RankSet(C)\) the set of $m!$ rankings over \(C\). Given a ranking \(r\) and two candidates \(c\) and \(c'\), we write \(c \succ_r c'\) if \(c\) is in a higher position than \(c'\) in \(r\). 
Given a ranking \(r\) and a candidate \(c\), \( \rk(c,r) \) denotes the rank of \(c\) in \(r\). For instance, \(\rk(c,r_v) = 1\) if \(c\) is the preferred candidate of voter \(v\) (the candidate ranked highest in $r_v$). Given a ranking \(r\) and a set \(S\subseteq C\), we define \(r_S\) as the restriction of \(r\) to \(S \) and \(\tp_r(S)\) as the top choice (i.e., preferred candidate) in \(S\) according to \(r\). Similarly, given a preference profile \(\PP\) and a set \(S\subseteq C\), we define \(\PPr{S}\) as the restriction of \(\PP\) to \(S \).  
Lastly, we denote by \(\mathtt{tail}_k(r)\) (resp. \(\mathtt{head}_k(r)\)) the subranking compounded of the \(k\) least (resp. most) preferred candidates in \(r\).

We are interested in SWFs which, given a preference profile \(\PP\), should return a consensus ranking which yields a suitable compromise between the preferences in \(\PP\). 
One of the most well-known SWFs is the \emph{Kemeny rule}, which selects a ranking \(r\) with minimal Kendall tau distance to \(\PP\). We recall:

\revised{
\begin{definition}
The Kendall tau distance between two rankings $r$ and $r'$ is defined by
\[
 \KendallTau(r,r') =  \sum_{(c,c')\in C^2} \disagree_{c,c'}(r,r')
\]
where \(\disagree_{c,c'}(r,r') = 1\) if $c \succ_r c'$ and $c' \succ_{r'} c$, and 0 otherwise.
\end{definition}
}
\revised{Stated differently, \(\KendallTau\) 
measures the distance between two rankings by the number of pairwise disagreements between them. The distance $\KendallTau(r,\PP)$ between a ranking \(r\) and a preference profile \(\PP\) is then obtained by summation:}
\revised{\begin{align*}
    \KendallTau(r,\PP) & = \sum_{r'\in \PP} \KendallTau(r,r').
\end{align*}
}

However, the Kendall tau distance only takes into account pairwise comparisons, which may entail counterintuitive results as illustrated by Example~\ref{ex:Baldiga}. To address this issue, the Kendall tau distance can be generalized to take into consideration disagreements on sets of \revised{cardinality} greater than two. 
Given a set \(S\subseteq C\) and $t \le m$, we denote by \(\SubsetSet{t}(S)\) the set of subsets of \(S\) of \revised{cardinality} lower than or equal to \(t\), i.e., \(\SubsetSet{t}(S)\!= \!\{ S'\subseteq S \text{ s.t. } |S'|\le t \}\). When \(S\) is not specified, it is assumed to be \(C\), i.e., \(\SubsetSet{t}\!=\! \SubsetSet{t}(C)\).  
\revised{\begin{definition}
Let $k\!\ge\!2$ be an integer. The \(k\)-wise Kendall tau distance \(\EKendallTau{k}\) between \(r\) and \(r'\) is defined by:
\begin{align*}
    \EKendallTau{k}(r,r') & =  \sum_{S \in \SubsetSet{k}} \disagree_S(r,r')
\end{align*}
where $\disagree_S(r,r')\!=\!1$ if \(\tp_{r}(S)\neq \tp_{r'}(S)\), and 0 otherwise.
\end{definition}}

In other words, \(\EKendallTau{k}\) measures the distance between two rankings by the number of top-choice disagreements on sets of \revised{cardinality} lower than or equal to \(k\). \revised{It is not hard to see that \(\EKendallTau{k}(r,r')\) can also be computed by using the following formula:
\begin{align}
    \EKendallTau{k}(r,r') &= \sum_{\revtwo{(c,c') \in C^2}}
    \disagree_{c,c'}(r,r')  ~|\SubsetSet{k-2}(\bel_{c}(r)\cap \bel_{c'}(r')| \nonumber \\
    &= \sum_{\revtwo{(c,c') \in C^2}}
    \disagree_{c,c'}(r,r')
    \sum_{i=0}^{k-2}
    \binom{|\bel_{c}(r)\cap \bel_{c'}(r')|}{i} \label{eq:deltaKT}
\end{align}
where \(\bel_c(r) = \{ x \in C \text{ s.t. } c \succ_r x \} \) is the set of candidates that are ranked below \(c\) in \(r\). Formula~\ref{eq:deltaKT} amounts to counting, for any pair \(\{c,c'\}\) of candidates such that \(c \succ_r c'\) and \(c' \succ_{r'} c\), the number of sets in \(\SubsetSet{k}\) on which there is a disagreement because the top choice is \(c\) for \(r\) while it is \(c'\) for \(r'\). Such sets are of the form $S\cup\{c,c'\}$, where $S \in \SubsetSet{k-2}(\bel_c(r)\cap \bel_{c'}(r'))$, otherwise $c$ and $c'$ would not be the top choices. Hence the formula.}

\medskip

Several observations can be made regarding \(\EKendallTau{k}\). Firstly, the following result states that \(\EKendallTau{k}\) has all the properties of a distance: 

\begin{proposition} \label{prop:Distance}
The function \(\EKendallTau{k}\) has all the properties of a distance: non-negativity, identity of indiscernibles, symmetry and triangle inequality.
\end{proposition}
\begin{proof}
Non-negativity and symmetry are obvious from the definition of \(\EKendallTau{k}\). It also verifies identity of indiscernibles: if \(\EKendallTau{k}(r,r')\!=\!0\), then rankings \(r\) and \(r'\) must in particular agree on each pairwise comparison, hence \(\KendallTau(r,r')\!=\!0\) and \(r\!=\!r'\) because \(\KendallTau\)  verifies identity of indiscernibles. Lastly, triangular inequality comes from the fact that given three rankings \(r_1\), \(r_2\) and \(r_3\) and a set \(S\), \HG{\(\disagree_{S}(r_1,r_3) \le \disagree_{S}(r_1,r_2) + \disagree_{S}(r_2,r_3)\).} 
\end{proof}

Secondly, as mentioned in the introduction, we have $\EKendallTau{2}\!=\!\KendallTau$. Thirdly and maybe most importantly, \(\EKendallTau{k}(r,r')\) can be computed in polynomial time in the number $m$ of candidates:
\begin{proposition}
\label{prop:Complexity}
Given two rankings \(r\) and \(r'\), \(\EKendallTau{k}(r,r')\) can be computed in \(O(m^3)\) by using Formula~\ref{eq:deltaKT}.
\end{proposition}
\begin{proof}
We prove the \(O(m^3)\) complexity of the method. First note that the computation of all binomial coefficients \(\binom{p}{i}\) for \(i\!\in\!\{0,\ldots,k\!-\!2\}\) and \(p\!\in\!\{i,\ldots,m\!-\!2\}\) can be performed in \(O(mk)\) thanks to Pascal's formula \( \binom{p}{i}\!+\! \binom{p}{i+1}\!=\!\binom{p+1}{i+1}\). Then the computation of the sums \(\sum_{i=0}^{k-2} \binom{p}{i}\) for \(p \in \{0,\ldots,m\!-\!2\}\) can also be computed in \(O(mk)\). For each pair \(\{c,c'\}\) of candidates \HG{such that $\tp_{r}(\{c,c'\}) = c$ and $\tp_{r'}(\{c,c'\}) = c'$}, the computation of \(|\bel_c(r)\cap \bel_{c'}(r')|\) can be performed in \(O(m)\). As there are \HG{at most} \(O(m^2)\) such pairs, the overall complexity of the method is \(O(m^3 + mk) = O(m^3)\). 
\end{proof}

The distance \(\EKendallTau{k}\) induces a new SWF, the \(k\)-wise Kemeny rule, which, given a profile \(\PP\), returns a ranking \(r\) with minimal distance \(\EKendallTau{k}\) to \(\PP\), where:
\begin{align*}
    \EKendallTau{k}(r,\PP) & = \sum_{r'\in \PP} \EKendallTau{k}(r,r'). 
\end{align*}

Note that this coincides with the rule we used in the introduction, by commutativity of addition:
\begin{align*}
\sum_{r'\in \PP} \sum_{S \in \SubsetSet{k}} \disagree_{S}(r,r')=  \sum_{S \in \SubsetSet{k}} \sum_{r'\in \PP} \disagree_{S}(r,r'). 
\end{align*}

Determining a consensus ranking for this rule \HG{induces an optimization problem that we term} 
the \(k\)-wise Kemeny Aggregation Problem (\(k\)-KAP for short).

\medskip

\noindent\fbox{\parbox{0.98\columnwidth}{
\(k\)-\textbf{WISE KEMENY AGGREGATION PROBLEM ($k$-KAP)}\\
\emph{INSTANCE:} A profile \(\PP\) with $n$ voters and $m$ candidates.\\
\emph{SOLUTION:} \HG{A ranking $r$ of the $m$ candidates minimizing $\EKendallTau{k}(r,\PP)$.}
}}\\

\HG{When discussing the complexity of the problem, we may also refer to the decision version of the problem, \(k\)-KAP-DEC:}

\medskip

\noindent\HG{\fbox{\parbox{0.98\columnwidth}{
\(k\)-\textbf{KAP}-\textbf{DEC}\\
\emph{INSTANCE:} A profile \(\PP\) with $n$ voters and $m$ candidates and a threshold $\tau \in \mathbb{N}$.\\
\emph{QUESTION:} Does there exist a ranking \(r\) of the \(m\) candidates such that \(\EKendallTau{k}(r,\PP)\le \tau\)?
}}\\
}

\subsection{Related Work}
Several other variants of the Kemeny rule have been proposed in the literature, either to obtain  generalizations able to deal with partial or weak orders~\citep{dwork2001rank,zwicker2018cycles}, to penalize more some pairwise disagreements than others~\citep{kumar2010generalized}, or to account for candidates that may become unavailable after voters express their preferences~\citep{baldiga2010choice,lu2010unavailable}. 

Despite its popularity, the Kemeny rule has received several criticisms.  One of them is that the Kendall tau distance counts equally the disagreements on every pair of candidates. This property is undesirable in many settings. For instance, with a web search engine, a disagreement on a pair of web pages with high positions in the considered rankings should have a higher cost than a disagreement on pairs of web pages with lower ones. This drawback motivated the introduction of weighted Kendall tau distances by Kumar and Vassilvitskii \cite{kumar2010generalized}. 

\subsubsection{Comparison between the \(k\)-wise and the weighted Kendall tau distances}

As mentioned above, Kumar and Vassilvitskii proposed that disagreements on highly ranked candidates be more costly than disagreements on lowly ranked ones. To achieve this, they defined a position-weighted version of Kendall tau, denoted by \(\WKendallTau{w}\), where an inversion of the two candidates at positions \(i\) and \(i-1\) has a cost \(w_i\). For convenience, a cost \(w_1 = 1\) is also defined. Given the costs \(w_i\), one can then measure the average swap-cost of moving a candidate from position \(i\) to \(j\) by computing the ratio \(\frac{p_i - p_j}{i-j}\) where \(p_i = \sum_{j=1}^{i} w_j\). This observation motivated the following definition for a position-weighted version of Kendall tau \citep{kumar2010generalized}:
\revised{\begin{definition} Let $w\!=\!(w_1,\ldots,w_m)$ be a vector of $m$ weights $w_i\!>\!0$. The position-weighted Kendall tau distance \(\WKendallTau{w}\) between $r$ and $r'$ is defined by: 
\[
\WKendallTau{w}(r,r') = \sum_{(c,c') \in C^2} \disagree_{c,c'}(r,r') \,\overline{p}(r,r',c)\overline{p}(r,r',c')
\] 
where \(\overline{p}(r,r',c) = \displaystyle\frac{p_{\rk(c,r)} - p_{\rk(c,r')}}{\rk(c,r) - \rk(c,r')}\) and \(\overline{p}(r,r',c) = 1\) if \(\rk(c,r) = \rk(c,r')\).
\end{definition}}
Put another way, if two rankings \(r\) and \(r'\) disagree on a pairwise comparison between two candidates \(c\) and \(c'\), then the cost of this disagreement is weighted by the product of the average swap-cost of moving \(c\) from its position in \(r\) to its position in \(r'\) with the average swap-cost of moving \(c'\) from its position in \(r\) to its position in \(r'\). Note that if \(w_i=1\) for all \(i\) in \(\{1,\ldots,m\}\), one obtains the usual Kendall tau distance.

Both the position-weighted Kendall tau distance and the \(k\)-wise Kendall tau distance can be used in order to penalize more strongly disagreements on candidates with high ranks (i.e., candidates that appear near the top of the ranking). For the \(k\)-wise Kendall tau distance, this property results from the fact that for a pair \(\{c,c'\}\) such that $\tp_{r}(\{c,c'\})\!\neq\!\tp_{r'}(\{c,c'\})$ the number of resulting subsets $S$ for which $\tp_{r}(S)\!\neq\!\tp_{r'}(S)$ is all the larger as $c$ and $c'$ are ranked high in $r$ and $r'$. Note however that the position-weighted Kendall tau distance requires to specify the \(m-1\) parameters $w_2,\ldots,w_m$, the tuning of which does not seem to be obvious. In comparison, the \(k\)-wise Kendall tau distance only requires to choose the value of \(k\), from which the cost of each disagreement on a pair \(\{c,c'\}\) of candidates is naturally entailed: it corresponds to the number of subsets of $C$ of size less than or equal to $k$ for which the top choice in $r$ is $c$ while the top choice in $r'$ is $c'$ (see Section~\ref{sec:Preliminaries}, \HG{Equation~\ref{eq:deltaKT}} for the formal expression of swap-costs according to \(k\)). 

Let us illustrate with the following example that the \(k\)-wise Kendall tau distance is also well suited to penalize more the disagreements involving alternatives at the top of the input rankings.

\begin{example}
\label{ex:KT}
Consider rankings $r_1,r_2,r_3$ defined by $c_1\!\succ_{r_1}\!c_2\!\succ_{r_1}\!c_3$, $c_1\!\succ_{r_2}\!c_3\! \succ_{r_2}\!c_2$, and  $c_2\!\succ_{r_3}\!c_1\! \succ_{r_3}\!c_3$. 
We have $\KendallTau(r_1,r_2)\!=\!\KendallTau(r_1,r_3)\!=\!1$ while $\EKendallTau{3}(r_1,r_2)\!=\!1\!<\!2\!=\!\EKendallTau{3}(r_1,r_3)$ because $r_1$ and $r_3$ disagree on both subsets 
$\{c_1,c_2\}$ and $\{c_1,c_2,c_3\}$. Put another way, $\EKendallTau{3}(r_1,r_3) > \EKendallTau{3}(r_1,r_2)$ because $r_1$ and $r_3$ disagree on their top-ranked alternatives whereas $r_1$ and $r_2$ disagree on the alternatives ranked in the last places.
\end{example}

\subsection{Comparison with aggregation models where candidates may become unavailable}

The two works closest to ours are related to another extension of the Kemeny rule. This extension considers a setting in which, besides the fact that voters have preferences over a set \(C\), the election will in fact occur on a subset \(S\subseteq C\) drawn according to a probability distribution \citep{baldiga2010choice,lu2010unavailable}.  
The optimization problem considered is then to find a consensus ranking $r$ which minimizes, in expectation, the number of voters' disagreements with the chosen candidate in \(S\) (a voter $v$ disagrees if $\tp_{r_v}(S)\!\neq\! \tp_{r}(S)$). The differences between the work of Baldiga  and Green \cite{baldiga2010choice} and the one of Lu and Boutilier \cite{lu2010unavailable} \revised{are then} twofold. 
Firstly, while Baldiga and Green mostly focused on the  axiomatic properties of this aggregation procedure, the work of Lu and Boutilier has more of an algorithmic flavor.
Secondly, while Baldiga and Green mostly study a setting in which the probability \(\mathbb{P}(S)\) of \(S\) is only dependent on its cardinality (i.e., \(\mathbb{P}(S)\) is only a function of \(|S|\)), Lu and Boutilier study a setting that can be viewed as a special case of the former, where each candidate is absent of \(S\) independently of the others 
with a probability \(p\) (i.e., \(\mathbb{P}(S) = p^{|C\setminus S|}(1-p)^{|S|}\)). The Kemeny aggregation problem can be formulated in both settings, either by defining \(\mathbb{P}(S)\!=\!0\) for \(|S|\!\ge\!3\), or by defining a probability \(p\) that is ``sufficiently high'' w.r.t. the size of the instance \citep{lu2010unavailable}. Lu and Boutilier conjectured that the determination of a consensus ranking is NP-hard in their setting, designed an exact method based on mathematical programming, two approximation greedy algorithms and a \HG{polynomial-time approximation scheme}. 

Our model can be seen as a special case of the model of Baldiga and Green where the set \(S\) is drawn uniformly at random within the set of subsets of \(C\) of \revised{cardinality} smaller than or equal to a given constant \(k\!\ge\!2\). While it cannot be casted in the specific setting studied by Lu and Boutilier, our model is closely related and may be used to obtain new insights on their work.

\section{Aggregation with the \(k\)-wise Kemeny Rule\label{sec:CompAndExact}}
In this section, we investigate the axiomatic properties of the \(k\)-wise Kemeny rule, and then we turn to the algorithmic study of \(k\)-KAP.
\subsection{Axiomatic Properties of the \(k\)-wise Kemeny Rule\label{sec:AxProperties}}
Several properties of the \(k\)-wise Kemeny rule have already been studied by Baldiga and Green \cite{baldiga2010choice}, because their setting includes the \(k\)-wise Kemeny rule as a special case. Among other things, they showed that the rule is not \emph{Condorcet consistent}. That is to say, a Condorcet winner may not be ranked first in any consensus ranking even when one exists, as illustrated by Example~\ref{ex:kwise}. \HG{The example indeed shows that the \(k\)-wise Kemeny rule is not Condorcet consistent for $k\!=\!3$, as the Condorcet winner is not ranked first in the unique consensus ranking for the 3-wise Kemeny rule. Note that any $k > 3$ would yield the same consensus ranking, as there are only \(m\!=\!3\) candidates in the profile, hence the result holds for any \(k\!\ge\!3\). The example can be generalized to show that the result holds even if \(m\!\ge\!k\):}
\HG{\begin{proposition}
The \(k\)-wise Kemeny rule is not Condorcet consistent for any $k\ge 3$, even if $m \ge k$.
\end{proposition}
\begin{proof}
We provide an example similar to Example~\ref{ex:kwise}. Consider an election with $k$ candidates $c_1,c_2,\ldots, c_k$ and 100 voters, with:\\
-- 49 voters having preferences $r_1$ defined by $c_1 \succ c_2 \succ c_3 \succ c_4 \succ \ldots \succ c_k$;\\
-- 49 voters having preferences $r_2$ defined by $c_3 \succ c_2 \succ c_1 \succ c_4 \succ \ldots \succ c_k$;\\
-- 2 voters having preferences $r_3$ defined by $c_2 \succ c_1 \succ c_3 \succ c_4 \succ \ldots \succ c_k$.\\
Note that $c_2$ is a Condorcet winner and that the ranking $r_1$ has a $k$-wise Kemeny score worth $0 + 49 \times (2^{k-2} + 2^{k-3} + 2^{k-3}) + 2\times 2^{k-2}\!=\!100 \times 2^{k-2}$ while ranking $r_3$ has a score worth $49 \times 2^{k-2} + 49 \times (2^{k-2} + 2^{k-3}) + 0\!=\! 122.5 \times 2^{k-2}$. It is easy to convince oneself that $r_3$ is the best possible ranking if $c_2$ is ranked first. Indeed, candidates $c_1$, $c_2$ and $c_3$ are ranked higher than other candidates in all three rankings $r_1$, $r_2$ and $r_3$ and therefore should be placed at the top in a consensus ranking (see Lemma~\ref{lem:dominated set}). Additionally, $c_1$ and $c_3$ have symmetric positions in rankings $r_1$ and $r_2$, and $c_1$ is preferred to $c_3$ in $r_3$, thus $c_1$ should be placed before $c_3$ in a consensus ranking. The result follows.
\end{proof}}

The authors also show that the $k$-wise Kemeny rule is neutral, i.e., all candidates are treated equally, and that for $k \ge 3$ it is different from any positional method or any method that uses only the pairwise majority margins (among which is the standard Kemeny rule). 
We provide here some additional properties satisfied by the \(k\)-wise Kemeny rule: 
\begin{itemize}
   \item \emph{Monotonicity}: up-ranking cannot harm a winner; down-ranking cannot enable a loser to win. \HG{Let us state this axiom more formally. Let \(\RankSet^*_{\PP}\) denote the set of consensus rankings for preference profile \(\PP\). Then for any candidate $c\in C$ and profiles $\PP$ and $\PP'$ such that $\PP'$ can be obtained from $\PP$ by decreasing the position of $c$ in some ranking in $\PP$ (all other things being equal): $c \in \{\tp_r(C) : r \in \RankSet^*_{\PP'}\}$ implies $c \in \{\tp_r(C) : r \in \RankSet^*_{\PP}\}$ and $c \notin \{\tp_r(C) : r \in \RankSet^*_{\PP}\}$ implies $c \notin \{\tp_r(C) : r \in \RankSet^*_{\PP'}\}$.}
    \item \emph{Unanimity}: if all voters rank \(c\) before \(c'\), then \(c\) is ranked before \(c'\) in any consensus ranking.
    \item \emph{Reinforcement}: let \(\RankSet^*_{\PP}\) and \(\RankSet^*_{\PP'}\) denote the sets of consensus rankings for preference profiles \(\PP\) and \(\PP'\) respectively. If \( \RankSet^*_{\PP} \cap \RankSet^*_{\PP'} \neq \emptyset\) and \(\PP''\) is the profile obtained by concatenating \(\PP\) and \(\PP'\), then \(\RankSet^*_{\PP''} = \RankSet^*_{\PP} \cap \RankSet^*_{\PP'}\).
\end{itemize}
\revised{Examples of voting rules for which the monotonicity property does not hold are \emph{plurality with run-off} (a two-round election system where, if some candidate is top ranked by a majority of the voters, it wins in round 1; otherwise, round 2 consists of the majority rule applied to the two candidates with highest plurality score in round 1) and \emph{single transferable vote} (at
each stage, the candidate with lowest plurality score is dropped from all votes and each vote for which this candidate was top ranked is transferred to the next remaining candidate in the ranking;
at the first stage for which some candidate $c$ sits atop a majority of the votes, $c$ is
declared the winner). For more details, the reader may refer to the book chapter by Zwicker~\cite{zwicker2016introduction}. The unanimity property is obviously desirable for any reasonable voting rule. Finally, the reinforcement property has been introduced by Young \cite{young1974axiomatization}, originally calling it \emph{consistency}, for the axiomatization of Borda's rule viewed as a \emph{social choice function}, i.e., returning a (set of) winning candidate(s). Reinforcement states that a candidate elected by two disjoint electorates should remain a winning candidate if one merges the voters, and that a candidate elected by only one electorate is in some sense not as ``good'' as a candidate elected by both electorates. The adaptation of the property to \emph{social welfare functions}, i.e., functions returning a (set of) consensus ranking(s), has been proposed by \citeauthor{young1978consistent} \cite{young1978consistent} for an axiomatic characterization of Kemeny's rule: they show that it is the unique Condorcet consistent social welfare function that satisfy neutrality and reinforcement, where the neutrality property states that all candidates are treated equally (in the sense that permuting the names of the candidates in the input rankings result in the same permutation in the consensus rankings).} 

While the fact that the \(k\)-wise Kemeny rule satisfies \revised{neutrality and} reinforcement is quite obvious from its definition, the two following results state that the monotonicity and unanimity conditions also hold.

\begin{proposition} \label{prop:Monotonicity}
  The \(k\)-wise Kemeny rule satisfies monotonicity.
\end{proposition}
\begin{proof}
Let \(r\) be a ranking and \(v\) be a voter such that \(c\) and \(c'\) are consecutive in \(r_v\) and \(c' \succ_{r_v} c\). Let us denote by \(r_v^{c\leftrightarrow c'}\) the ranking obtained from \(r_v\) by switching the positions of \(c\) and \(c'\). Then, the sets \(S\) for which \(\tp_{r_v}(S)\!\neq\! \tp_{r_v^{c\leftrightarrow c'}}(S)\) are of the form \(\{c,c'\} \cup S'\) where \(S' \subseteq \bel_{c}(r_v)\). Furthermore, if such a set \(S\) contains a candidate \(c''\) such that \(c'' \succ_r \tp_r(\{c,c'\})\), then we will both have \(\tp_r(S) \neq \tp_{r_v}(S)\) and \(\tp_r(S) \neq t_{r_v^{c\leftrightarrow c'}}(S)\). Hence, the sets which account for the difference between \(\EKendallTau{k}(r,r_v)\) and \(\EKendallTau{k}(r,r_v^{c\leftrightarrow c'})\) are of the form \(\{c,c'\} \cup S'\) where \(S' \subseteq \bel_{c}(r_v) \cap \bel_{\tp_r(\{c,c'\})}(r)\).

More precisely, \revised{using Equation~\ref{eq:deltaKT}}, we obtain that  \(\EKendallTau{k}(r,r_v^{c\leftrightarrow c'})\) is equal to:
\begin{itemize}
\item \(\EKendallTau{k}(r,r_v) - \sum_{i=0}^{k-2}\binom{|\bel_{c'}(r_v)\cap \bel_{c}(r)|}{i}\) if \(c \succ_r c'\);
\item \(\EKendallTau{k}(r,r_v) + \sum_{i=0}^{k-2}\binom{
    |\bel_c(r_v)\cap \bel_{c'}(r)|}{i}\) if \(c' \succ_r c\).
\end{itemize}
Hence, \(\EKendallTau{k}(r,r_v)\) will decrease if \(r\) ranks \(c\) before \(c'\) and the decrease is maximal when \(c\) is ranked first in \(r\) (because it maximizes \(|\bel_{c'}(r_v)\cap \bel_{c}(r)|\)). Repeating this argument shows that no winner is harmed by up-ranking. Similarly, \(\EKendallTau{k}(r,r_v)\) will increase if \(r\) ranks \(c'\) before \(c\) and the increase is maximal when \(c'\) is ranked first in \(r\) (because it maximizes \(|\bel_c(r_v)\cap \bel_{c'}(r)|\)). Repeating this argument shows that no loser can win by down-ranking.  
\end{proof}

\begin{proposition} \label{prop:Unanimity}
  The \(k\)-wise Kemeny rule satisfies unanimity.
\end{proposition}
\begin{proof}
Let \(c,c'\in C\) be two candidates and \(\PP\) be a preference profile such that for all ranking \(r''\) in \(\PP\), \(c\succ_{r''} c'\). Let \(r\) be a ranking such that \(c'\succ_r c\), and \(r'\) be the ranking obtained from \(r\) by exchanging the positions of \(c'\) and \(c\). Moreover, let \(K\) denote the set of candidates between \(c\) and \(c'\) in \(r\).
Let us assume for the sake of contradiction that \(r\) minimizes \(\EKendallTau{k}(\cdot, \PP)\). We will prove that for any \(r''\in \PP\), \(\EKendallTau{k}(r',r'') < \EKendallTau{k}(r,r'')\). Let \(r''\in \PP\) and \(S\subseteq C\) such that \(\tp_{r'}(S) \neq \tp_{r''}(S)\) and \(\tp_{r}(S) = \tp_{r''}(S)\). Then \(S\) must contain either \(c\) or \(c'\), and does not contain any element ranked higher than \(c'\) in \(r\) because otherwise we would have \(\tp_{r}(S) = \tp_{r'}(S)\). This implies that either \(\tp_{r'}(S) = c\) (if \(c\in S\)) or \(\tp_{r}(S) = c'\) (if \(c'\in S\)). These situations are exclusive: there cannot be both \(c\) and \(c'\) in \(S\) as we cannot have \(\tp_{r''}(S) = c'\) if \(c\in S\). To sum up, there are two possibilities:
\begin{enumerate}
    \item \(\tp_{r'}(S) = c\) and \(\tp_{r''}(S) = \tp_{r}(S) = c''\) is in \(K\). This implies that \(c''\) is the second choice of \(r'\) in \(S\) and that \(c'' \succ_{r''} c'\) as \(c'' \succ_{r''} c\). In this case, necessarily, \(c' \not \in S\) and we consider \(S' = (S\setminus\{c\})\cup\{c'\}\).
    \item \(\tp_{r'}(S) \in K\) and \(\tp_{r''}(S) = \tp_{r}(S) = c'\). In this case, necessarily, \(c \not \in S\) and we consider \(S' = S\cup\{c\}\).
\end{enumerate}
In both cases, we obtain a set \(S'\) such that \(\tp_{r'}(S') = \tp_{r''}(S')\) and \(\tp_{r}(S') \neq \tp_{r''}(S')\). Note that any set \(S\) will induce a different \(S'\) and that \(\{c,c'\}\) is not one of these sets \(S'\).  As we also have \(\tp_{r'}(\{c,c'\}) = \tp_{r''}(\{c,c'\})\) and \(\tp_{r}(\{c,c'\}) \neq \tp_{r''}(\{c,c'\})\), this proves that for any \(r''\in \PP\), \(\EKendallTau{k}(r',r'') < \EKendallTau{k}(r,r'')\) and hence the claim. 
\end{proof}

Besides, the \(k\)-wise Kemeny rule does not satisfy \emph{Independence of irrelevant alternatives}, i.e., the relative positions of two candidates in a consensus ranking can depend on the presence of other candidates. Let us illustrate this point with the following example. 
\begin{example}
Considering the preference profile from Example~\ref{ex:Baldiga},  the only consensus ranking for \(\EKendallTau{3}\) is $c_1\!\succ\!c_2\!\succ \!c_3$. Yet, without \(c_3\) the only consensus ranking would be \(c_2\!\succ\!c_1\).
\end{example}

Lastly, note that there exists a noise model such that the \(k\)-wise Kemeny rule can be interpreted as a maximum likelihood estimator~\citep{conitzer2009preference}. In this view of voting, one assumes that there exists a ``correct'' ranking $r$, and each 
vote corresponds to a noisy perception of this
correct ranking. Consider the conditional probability measure $\mathbb{P}$ on \(\RankSet(C)\) defined by \(\mathbb{P}(r'|r)\!\propto\! e^{-\EKendallTau{k}(r,r')}\). It is easy to convince oneself that the \(k\)-wise Kemeny rule returns a ranking $r^*$ that maximizes \(\mathbb{P}(\PP|r^*)\!=\!\prod_{r'\in \PP} \mathbb{P}(r'|r^*) \) and is thus a maximum likelihood estimate of \(r\).

\subsection{Computational Complexity of \(k\)-KAP}
We now turn to the algorithmic study of \(k\)-KAP. After providing a hardness result, we will design an efficient Fixed Parameter Tractable (FPT) algorithm for parameter \(m\). 

While \(k\)-KAP is obviously NP-hard for \(k\!=\!2\) as it then corresponds to determining a consensus ranking w.r.t. the Kemeny rule, we strengthen this result by showing that \HG{\(k\)-KAP-DEC} is also \HG{NP-complete} for \emph{any} constant value \(k\!\ge\!3\). 
To prove this result, we first need two lemmas.

\begin{lemma}\label{lem:dominated set}
If the candidates in a  set \(S\subseteq C\) are ranked in the \(|S|\) last positions by \revised{all voters} and in the same order, then for any \(k\ge 2\), any consensus ranking w.r.t. the \(k\)-wise Kemeny rule has the same property.
\end{lemma}
\begin{proof}
This is a simple consequence of the unanimity property that is satisfied by the \(k\)-wise Kemeny rule. 
\end{proof}

\begin{lemma}\label{lem:power to factor}
For any \(p\in \mathbb{N}^*\) and \(\varepsilon\in (0,\frac{1}{2p})\) we have the following inequality: 
\[(1+ \varepsilon)^p < 1+ 2p\varepsilon\] 
\end{lemma}
\begin{proof}
We prove the claim by induction. It is obvious for \(p\!=\!1\).
Consider the claim true for \(p\!=\!k\), then for \(\varepsilon\in (0,\frac{1}{2(k+1)})\)
\begin{align*}
 (1\!+\!\varepsilon)^{k+1}\!&=\!(1+ \varepsilon)(1\!+\!\varepsilon)^k 
                <\! (1\!+\!\varepsilon)(1\!+\!2k\varepsilon) \\                  &=\!1\!+\! 2k\varepsilon\!+\!\varepsilon\!+\! 2k\varepsilon^2
            <\! 1\!+\!2(k+1)\varepsilon 
\end{align*} 
where the first inequality uses the induction hypothesis and the second inequality uses the fact that \(2k\varepsilon^2\!<\!\varepsilon\) because \(2k\varepsilon\!<\!1\) for \(\varepsilon \!\in\!(0,\frac{1}{2(k+1)})\). 
\end{proof}

We can now prove the hardness result, by using a reduction from \HG{\(2\)-KAP-DEC}.

\begin{theorem}
\label{thrm:hardness}
For any constant \(k\!\ge\!3\), \HG{\(k\)-KAP-DEC is NP-complete}, even if the number of voters equals 4 or if the average range of candidates \HG{is less than or equal to} 2 (where the range of a candidate \(c\) is defined by \(\max_{r\in \PP}\rk(c,r) - \min_{r\in \PP}\rk(c,r) +1\) and the average is taken over all candidates). 
\end{theorem}
\begin{proof}
\HG{Membership in NP follows from Proposition~\ref{prop:Complexity}.} 
We obtain our hardness result via a reduction from the standard Kemeny aggregation problem (\HG{2-KAP-DEC}), which is known to be \HG{NP-complete} \citep{bartholdi1989voting}. Consider a preference profile \(\PP\) with \(n\ge 1\) voters and \(m \ge k \ge 3\) candidates \HG{and an integer $\tau$. We wish to determine if there exists a ranking $r$ of the $m$ candidates such that $\KendallTau(r,\PP) \le \tau$. Stated otherwise, we wish to determine if the Kemeny score of a consensus ranking is lower than or equal to $\tau$}. Note that we can assume \(m\ge k\) as \(k\) is a constant and 2-KAP is fixed parameter tractable w.r.t. \(m\) \citep{betzler2009fixed}. 
We add to the problem \(\lambda\!=\! 4nm^4\) candidates \(c^*_1, \ldots, c^*_{\lambda}\) that are ranked last by all voters and in the same order, i.e., \(c^*_1 \succ_r \ldots \succ_r c^*_{\lambda}\) for all \(r\) in \(\PP\). We denote the resulting set of candidates by \(C'\) (i.e., \(C'\!=\!C\cup\{c^*_1, \ldots, c^*_{\lambda}\}\)) and the resulting preference profile by \(\PP'\). By using Lemma \ref{lem:dominated set}, we will restrict our attention to rankings that rank these additional voters last and in the same order as the voters, because they are the only possible consensus rankings. \HG{Lastly, in the resulting $k$-KAP-DEC instance, we set $\tau' = (1 + \sum_{i=1}^{k-2} \binom{4nm^4}{i}) (\tau + 1) - 1$. }

Given two such rankings $r$ and $r'$, we have \revised{by Equation~\ref{eq:deltaKT}} that \(\EKendallTau{k}(r,r')\) is equal to:
\revised{\begin{align*}
    &\sum_{(c,c')\in C^2} \disagree_{c,c'}(r,r')\sum_{i=0}^{k-2} \binom{|\bel_c(r_C)\cap \bel_{c'}(r'_C)|+\lambda}{i}\\
    =&\sum_{(c,c')\in C^2} \disagree_{c,c'}(r,r') \left(1+\sum_{i=1}^{k-2} \binom{|\bel_c(r_C)\cap \bel_{c'}(r'_C)|+\lambda}{i}\right)
\end{align*}}
because there is no disagreement for \(\{c,c'\} \not\subseteq C\) and $\binom{|\bel_c(r_C)\cap \bel_{c'}(r'_C)|+\lambda}{0}\!=\!1$. 

From \(0 \le |\bel_c(r_C)\cap \bel_{c'}(r'_C)| < m\), we deduce:
\begin{align*}
    \sum_{i=1}^{k-2}\!\binom{\lambda}{i} \!\le\!
    \sum_{i=1}^{k-2}\!\binom{|\bel_c(r_C)\cap \bel_{c'}(r'_C)|+\lambda}{i} \!<\!
    \sum_{i=1}^{k-2}\!\binom{m+\lambda}{i}.
\end{align*}
Consequently:
\begin{align*}
    \KendallTau(r_C,r'_C)\left(1+\sum_{i=1}^{k-2} \binom{\lambda}{i}\right) \le
    \EKendallTau{k}(r,r') <
    \KendallTau(r_C,r'_C)\left(1+\sum_{i=1}^{k-2} \binom{m+\lambda}{i}\right)
\end{align*}
\revised{because \(\displaystyle\sum_{(c,c') \in C^2} \disagree_{c,c'}(r,r')=\KendallTau(r_C,r'_C)\); from which we obtain:}
\begin{align} 
    \KendallTau(r_C,\PP)\left(1+\sum_{i=1}^{k-2} \binom{\lambda}{i}\right) \le
    \EKendallTau{k}(r,\PP') <
    \KendallTau(r_C,\PP)\left(1+\sum_{i=1}^{k-2} \binom{m+\lambda}{i}\right) \label{eq:ineq}
\end{align}
because \(\sum_{r'\in\PP'}\KendallTau(r_C,r'_C)\!=\!\KendallTau(r_C,\PP)\) and \(\sum_{r'\in\PP'}\EKendallTau{k}(r,r')\!=\!\EKendallTau{k}(r,\PP')\), provided that \(\mathtt{tail}_\lambda(r)\!=\!c^*_1\!\succ_r\!\ldots \!\succ_r\!c^*_{\lambda}\) and \(\mathtt{tail}_\lambda(r')\!=\!c^*_1\!\succ_{r'}\!\ldots \!\succ_{r'}\!c^*_{\lambda}\) \(\forall r'\!\in\!\PP'\).

Now, note that:
\begin{align} 
\binom{m+\lambda}{i} = \frac{\lambda! }{i!(\lambda - i)!} \frac{\prod_{j=1}^m(\lambda+j)}{\prod_{j=1}^m(\lambda-i+j)} 
= \binom{\lambda}{i} \prod_{j=1}^m \frac{\lambda+j}{\lambda+j-i} \label{eq:lambda_i}.
\end{align}

If one sets \(\lambda\!=\! 4nm^4\), the following inequalities hold:
\begin{align*}
    \frac{4nm^4+j}{4nm^4+j-i} &= 1+\frac{i}{4nm^4+j-i} \le 1+\frac{i}{4nm^4-i} \\&\le 1+\frac{2i}{4nm^4} \le 
    1+\frac{1}{2nm^3}
\end{align*}
where the second inequality follows from \(4nm^4/(4nm^4-i)\le 2\) for \(i\in \llbracket 1,m\rrbracket\). 
 
From Equation~\ref{eq:lambda_i}, we deduce then:
\begin{align*}
    \binom{m+4nm^4}{i} \le \binom{4nm^4}{i}(1+\frac{1}{2nm^3})^m \le \binom{4nm^4}{i} (1+\frac{1}{nm^2})
\end{align*}
where the second inequality follows from Lemma~\ref{lem:power to factor} with \(\varepsilon\!=\!\frac{1}{2nm^3}\) because $\frac{1}{2nm^3}\!<\!\frac{1}{2m}$ for $m\!\ge\!3$.
 
Coming back to Equation~\ref{eq:ineq}, this implies:
\begin{align*}
    \KendallTau(r_C,\PP) \le
    \frac{\EKendallTau{k}(r,\PP')}{1+\sum_{i=1}^{k-2} \binom{4nm^4}{i}} <
    \KendallTau(r_C,\PP)+\frac{1}{nm^2}\KendallTau(r_C,\PP)
\end{align*}
and therefore:
\begin{align*}
    \KendallTau(r_C,\PP) \le
    \frac{\EKendallTau{k}(r,\PP')}{1+\sum_{i=1}^{k-2} \binom{4nm^4}{i}} <
    \KendallTau(r_C,\PP)+1
\end{align*}
because \(\KendallTau(r_C,\PP) \le n\binom{m}{2} \le nm^2\). \HG{In particular, $\KendallTau(r_C,\PP) \le \tau$ for an integer $\tau$ iff 
$\EKendallTau{k}(r,\PP') \le (1 + \sum_{i=1}^{k-2} \binom{4nm^4}{i}) (\tau + 1) - 1 = \tau'$. This shows that $(\PP,\tau)$ is a yes instance of 2-KAP-DEC iff $(\PP',\tau')$ is a yes instance of $k$-KAP-DEC.}

It is known that \HG{2-KAP-DEC} is \HG{NP-complete} even if the number \(n\) of voters equals 4 \citep{dwork2001rank} and even if the average range of candidates equals 2 \citep{betzler2009fixed}. As the reduction above preserves the number of voters and decreases the average range of candidates, the same results hold for \HG{\(k\)-KAP-DEC}. 
\end{proof}

Although the above proof uses a conversion from the \(k\)-wise Kemeny rule to the standard Kemeny rule, note that it means in no way that both rules are equivalent. \revised{There indeed exist} instances with arbitrary large sets of candidates such that the \(k\)-wise Kemeny rule differs from the Kemeny rule. Consider for instance the election described in Example \ref{ex:Baldiga}. The only \(k\)-wise Kemeny consensus ranking is \(c_1\succ c_2 \succ c_3\), while the only pairwise Kemeny consensus ranking is \(c_2\succ c_3 \succ c_1\). Now modify the preference profile of the previous election by adding candidates \(c_4\) to \(c_k\), with \(k\) arbitrarily large such that for any ranking \(r\) in the new preference profile, \(\mathtt{head}_{k-3}(r) = c_k \succ c_{k-1} \succ \ldots \succ c_4\). With this new preference profile, the only \(k\)-wise Kemeny consensus ranking is \(c_k \succ c_{k-1} \succ \ldots \succ c_4 \succ c_1\succ c_2 \succ c_3\), while the only pairwise Kemeny consensus ranking is \(c_k \succ c_{k-1} \succ \ldots \succ c_4 \succ c_2\succ c_3 \succ c_1\). 

\medskip

Despite Theorem~\ref{thrm:hardness}, \(k\)-KAP is obviously FPT w.r.t. the number $m$ of candidates, by simply trying the \(m!\) rankings in \(\RankSet(C)\).  
We now design a dynamic programming procedure which significantly improves this time complexity.

\begin{proposition}
If $r^*$ is an optimal ranking for $k$-KAP, then $\EKendallTau{k}(r^*,\PP)\!=\!d_{\mathtt{KT}}^k(C)$, where, for any subset $S\subseteq C$, $d_{\mathtt{KT}}^k(S)$ is defined by the recursive relation:
\begin{align}
    d_{\mathtt{KT}}^k(S) & = \min_{c\in S}  [d_{\mathtt{KT}}^k(S\setminus \{c\}) \notag\\ & + \sum_{r \in \PPr{S}} \sum_{c'\succ_{r} c} \sum_{i=0}^{k-2}  \binom{|S|-\rk(c',r)-1}{i}]
    \label{eq:progdyn}\\
    d_{\mathtt{KT}}^k(\emptyset) & = 0. \notag
\end{align}
\end{proposition}

\begin{proof}
Given \(S\!\subseteq\!C\) and \(c\!\in\!S\), let us define \(\RankSet_c(S)\) as \(\{r\!\in\! \RankSet(S) \mbox{ s.t. } \tp_r(S)\!=\!c\}\).
The set \(\SubsetSet{k}(S)\) can be partitioned into  \(\SubsetSet{k}_c(S)\!=\!\{S'\in \SubsetSet{k}(S) \text{ s.t. } c \in S'\}\) and \(\SubsetSet{k}_{\overline{c}}(S)\!=\!\{S'\in \SubsetSet{k}(S) \text{ s.t. } c \not\in S'\}\!=\! \SubsetSet{k}(S\setminus\{c\})\). Given a preference profile $\PP$ over $C$ and a ranking \(\hat{r}\!\in\!\RankSet_c(S)\), the summation defining \(\EKendallTau{k}(\hat{r},\PPr{S})\) breaks down as follows:
\begingroup
\allowdisplaybreaks
\begin{align}
   &\EKendallTau{k}(\hat{r},\PPr{S}) = \sum_{r \in \PPr{S}} \sum_{S'\in \SubsetSet{k}(S)}    \disagree_{S'}(\hat{r},r)\notag\\
   & = \EKendallTau{k}(\hat{r}_{S\setminus\{c\}},\PP_{S\setminus\{c\}}) + \sum_{r \in \PPr{S}} \sum_{S'\in \SubsetSet{k}_{c}(S)}  \disagree_{S'}(\hat{r},r). \label{eq:initPD}
\end{align}
\endgroup
Using the same reasoning as \revised{in Equation~\ref{eq:deltaKT}} on page~\pageref{eq:deltaKT}, the second summand in Equation~\ref{eq:initPD} can be rewritten as follows:
\begin{equation*}
    \sum_{r \in \PPr{S}} \sum_{c' \succ_r c} \sum_{i=0}^{k-2} \binom{|\bel_c(\hat{r})\cap \bel_{c'}(r)|}{i}
\end{equation*}
because \HG{\(\tp_{\hat{r}}(S') = c\)} for all \(S'\!\!\in\!\! \SubsetSet{k}_{c}(S)\). Note that \(\bel_c(\hat{r}) \!=\! S\!\setminus\!\{c\}\) and \(\bel_{c'}(r)\!=\!\{c''\!\in\!S \mbox{ s.t. } c'\!\succ_{r}\!c''\} \subseteq\!S\), thus \(|\bel_c(\hat{r})\cap \bel_{c'}(r)|\!=\!|S|-\rk(c',r)-1\). Hence, $\EKendallTau{k}(\hat{r},\PPr{S})$ is equal to:
\begin{equation} 
 \EKendallTau{k}(\hat{r}_{S\setminus\{c\}},\PP_{S\setminus\{c\}}) + \sum_{r \in \PPr{S}} \sum_{c'\succ_{r} c} \sum_{i=0}^{k-2}  \binom{|S|-\rk(c',r)-1}{i}. \label{eq:initPDv2}
\end{equation}
Consider now a ranking \(r^* \in \RankSet(S)\) such that \(\EKendallTau{k}(r^*,\PPr{S})=\min_{r \in \RankSet(S)} \EKendallTau{k}(r,\PPr{S})\). We have:
\begin{align*}
    \EKendallTau{k}(r^*,\PPr{S})
    & = \min_{c \in S} \min_{\hat{r} \in \RankSet_c(S)} \EKendallTau{k}(\hat{r},\PPr{S})\\
    & = \min_{c \in S} \Bigl((\min_{\hat{r} \in \RankSet(S\setminus\{c\})} \EKendallTau{k}(\hat{r},\PPr{S\setminus\{c\}}))\\
    & + \sum_{r \in \PPr{S}} \sum_{c'\succ_{r} c} \sum_{i=0}^{k-2}  \binom{|S|-\rk(c',r)-1}{i} \Bigr)
\end{align*}
because the second summand in Equation~\ref{eq:initPDv2} does not depend on $\hat{r}$ (it only depends on $c$, which is the argument of the first min operator). If one denotes \(\min_{r \in \RankSet(S)} \EKendallTau{k}(r,\PPr{S})\) by \(d_{\mathtt{KT}}^k(S)\), one obtains Equation~\ref{eq:progdyn}. This concludes the proof. 
\end{proof}
A candidate \(c\!\in\!S\) that realizes the minimum in Equation~\ref{eq:progdyn} can be ranked in first position in an optimal ranking for \(\PP_S\). Once \(d_{\mathtt{KT}}^k(S)\) is computed for each \(S\!\subseteq\!C\), a ranking $r^*$ achieving the optimal value \(d_{\mathtt{KT}}^k(C)\) can thus be determined recursively starting from \(S\!=\!C\). The complexity of the induced dynamic programming method is \(O(2^mm^2n)\) as there are \(2^m\) subsets $S \subseteq C$ to consider and each value \(d_{\mathtt{KT}}^k(S)\) is computed in \(O(m^2n)\) by Equation~\ref{eq:progdyn}. The min operation is indeed performed on $m$ values and the sum \(\sum_{c'\succ_{r} c} \sum_{i=0}^{k-2}  \binom{|S|-\rk(c',r)-1}{i}\) is computed incrementally in \(O(m)\), which entails an \(O(mn)\) complexity for the second summand in Equation~\ref{eq:progdyn} (the $n$ factor is due to the sum over all \(r\!\in\!\PPr{S}\)). The computation of binomial coefficients \(\binom{p}{i}\) for \(i\!\in\!\{0,\ldots,k-2\}\) and \(p\!\in\!\{i,\ldots,m-2\}\) is performed in \(O(mk)\) in a preliminary step thanks to Pascal's formula.

\section{The \(k\)-Wise Majority Digraph\label{sec:Graph}}
We now propose and investigate a \(k\)-wise counterpart of the pairwise majority digraph, that will be used in a preprocessing procedure for \(k\)-KAP.

As stated in the introduction, the pairwise Kemeny rule is strongly related to the \emph{pairwise majority digraph}. We denote by \(\MG_{\PP}\) the pairwise majority digraph associated to profile \(\PP\). We recall that in this digraph, there is one vertex per candidate, and there is an arc from candidate \(c\) to candidate \(c'\) if a strict majority of voters prefers \(c\) to \(c'\). In the weighted pairwise majority digraph, each arc \((c,c')\) is weighted by  \(w_{\PP}(c,c') \coloneqq |\{r\in \PP \text{ s.t. } c \succ_r c'\}| - |\{r\in \PP \text{ s.t. } c' \succ_r c\}|\).  

\begin{example} \label{ex:MG}
Consider a profile \(\PP\) with 10 voters and 6 candidates such that:\\ 
-- 4 voters have preferences \(c_1 \succ c_2 \succ c_4 \succ c_3 \succ c_5 \succ c_6\);\\
-- 4 voters have preferences \(c_1 \succ c_3 \succ c_2 \succ c_4 \succ c_5 \succ c_6\);\\
-- 1 voter has preferences \(c_6 \succ c_1 \succ c_2 \succ c_4 \succ c_3 \succ c_5\);\\
-- 1 voter has preferences \(c_6 \succ c_1 \succ c_4 \succ c_3 \succ c_2 \succ c_5\).\\
The weighted pairwise majority digraph \(\MG_{\PP}\) is displayed on the left of Figure~\ref{fig:extendedMajorityGraph}.
\end{example}

\begin{figure}[hbt]
\begin{minipage}[c]{.46\linewidth} 
\begin{center}
\scalebox{1}{\begin{tikzpicture}
\clip (0,0) rectangle (6,6);
\Vertex[x=5.650,y=2.000,color=white,label=\large $c_1$,fontcolor=black]{1}
\Vertex[x=4.325,y=3.295,color=white,label=\large $c_2$,fontcolor=black]{2}
\Vertex[x=1.675,y=3.295,color=white,label=\large $c_3$,fontcolor=black]{3}
\Vertex[x=0.350,y=2.000,color=white,label=\large $c_4$,fontcolor=black]{4}
\Vertex[x=1.675,y=0.705,color=white,label=\large $c_5$,fontcolor=black]{5}
\Vertex[x=4.325,y=0.705,color=white,label=\large $c_6$,fontcolor=black]{6}
\Edge[,bend=-8.531,label=\footnotesize 10,distance=0.60,Direct](1)(2)
\Edge[,bend=-8.531,label=\footnotesize 10,distance=0.60,Direct](1)(3)
\Edge[,bend=-8.531,label=\footnotesize 10,distance=0.70,Direct](1)(4)
\Edge[,bend=-8.531,label=\footnotesize 10,distance=0.68,Direct](1)(5)
\Edge[,bend=8.531,label=\footnotesize 6,distance=0.65,Direct](1)(6)
\Edge[,bend=-8.531,label=\footnotesize 8,distance=0.65,Direct](2)(4)
\Edge[,bend=-8.531,label=\footnotesize 10,distance=0.65,Direct](2)(5)
\Edge[,bend=-8.531,label=\footnotesize 6,distance=0.26,Direct](2)(6)
\Edge[,bend=-8.531,label=\footnotesize 10,distance=0.65,Direct](3)(5)
\Edge[,bend=-8.531,label=\footnotesize 6,distance=0.58,Direct](3)(6)
\Edge[,bend=8.531,label=\footnotesize 2,distance=0.65,Direct](4)(3)
\Edge[,bend=-8.531,label=\footnotesize 10,distance=0.60,Direct](4)(5)
\Edge[,bend=-8.531,label=\footnotesize 6,distance=0.65,Direct](4)(6)
\Edge[,bend=-8.531,label=\footnotesize 6,distance=0.65,Direct](5)(6)
\end{tikzpicture}}  
\end{center}
   \end{minipage}
\hfill
   \begin{minipage}[c]{.46\linewidth}
   \begin{center}
\scalebox{1}{\begin{tikzpicture}
\clip (0,0) rectangle (6,6);
\Vertex[x=5.650,y=2.000,color=white,label=\large $c_1$,fontcolor=black]{1}
\Vertex[x=4.325,y=3.295,color=white,label=\large $c_2$,fontcolor=black]{2}
\Vertex[x=1.675,y=3.295,color=white,label=\large $c_3$,fontcolor=black]{3}
\Vertex[x=0.350,y=2.000,color=white,label=\large $c_4$,fontcolor=black]{4}
\Vertex[x=1.675,y=0.705,color=white,label=\large $c_5$,fontcolor=black]{5}
\Vertex[x=4.325,y=0.705,color=white,label=\large $c_6$,fontcolor=black]{6}
\Edge[,bend=-8.531,label=\footnotesize 48,distance=0.60,Direct](1)(2)
\Edge[,bend=-8.531,label=\footnotesize 48,distance=0.60,Direct](1)(3)
\Edge[,bend=-8.531,label=\footnotesize 48,distance=0.75,Direct](1)(4)
\Edge[,bend=-8.531,label=\footnotesize 48,distance=0.70,Direct](1)(5)
\Edge[,bend=+8.531,label=\footnotesize 30,distance=0.60,Direct](1)(6)
\Edge[,bend=-8.531,label=\footnotesize 1,distance=0.65,Direct](2)(3)
\Edge[,bend=-8.531,label=\footnotesize 28,distance=0.65,Direct](2)(4)
\Edge[,bend=-8.531,label=\footnotesize 32,distance=0.65,Direct](2)(5)
\Edge[,bend=-8.531,label=\footnotesize 20,distance=0.28,Direct](2)(6)
\Edge[,bend=-8.531,label=\footnotesize 1,distance=0.65,Direct](3)(4)
\Edge[,bend=-8.531,label=\footnotesize 27,distance=0.65,Direct](3)(5)
\Edge[,bend=-8.531,label=\footnotesize 16,distance=0.57,Direct](3)(6)
\Edge[,bend=-8.531,label=\footnotesize 4,distance=0.65,Direct](4)(3)
\Edge[,bend=-8.531,label=\footnotesize 25,distance=0.60,Direct](4)(5)
\Edge[,bend=-8.531,label=\footnotesize 14,distance=0.10,Direct](4)(6)
\Edge[,bend=-8.531,label=\footnotesize 6,distance=0.65,Direct](5)(6)
\Edge[,bend=-8.531,label=\footnotesize 2,distance=0.65,Direct](6)(5)
\end{tikzpicture}}
   \end{center}
   \end{minipage}
\caption{\label{fig:extendedMajorityGraph}Weighted \(k\)-wise majority digraph in Example~\ref{ex:MG} for $k=2$ (left) and $k=3$ (right).}
\end{figure}

From \(\MG_{\PP}\), we can define a set of \emph{consistent} rankings:

\begin{definition}
Let \(\mathcal{G}\) be a digraph whose vertices correspond to the candidates in \(C\). Let \(B_1(\mathcal{G}),\ldots, B_{\sigma(\mathcal{G})}(\mathcal{G})\) denote the subsets of \(C\) corresponding to the Strongly Connected Components (SCCs) of \( \mathcal{G} \), and \(\TopSortSet{\mathcal{G}}\) denote the set of linear orders \(<_{\mathcal{G}}\) on \(\{1, \ldots, \sigma(\mathcal{G})\}\) such that if there exists an arc \((c,c')\) from \(c \!\in\!B_i(\mathcal{G})\) to \(c'\!\in\! B_j(\mathcal{G})\) then  \(i\!<_{\mathcal{G}}\!j\). 
Given \(<_{\mathcal{G}}\in \!\TopSortSet{\mathcal{G}}\), we say that a ranking \(r\) is \emph{consistent} with \(<_{\mathcal{G}}\) if the candidates in \(B_i\) are ranked before the ones of \(B_j\) when \(i\!<_{\mathcal{G}}\!j\). 
\end{definition}

The following result states that, for any  \(<_{\MG_{\PP}} \in\! \TopSortSet{\MG_{\PP}}\), there exists a consensus ranking for \(\KendallTau\) among the rankings consistent with \(<_{\mathcal{G}_{\PP}}\).

\begin{theorem}[Theorem 16 in \cite{charon2010updated}, by Charon and Hudry] \label{thrm:charon}
Let \(\PP\) be a profile over \(C\) and assume that the SCCs of \(\MG_{\PP}\) are numbered according to a linear order \(<_{\MG_{\PP}} \in\! \TopSortSet{\MG_{\PP}}\). Consider the ranking \(r^*\), consistent with \(<_{\MG_{\PP}}\), obtained by the concatenation of rankings \(r_1^*,\ldots,r_{t(\MG_{\PP})}^*\) 
 where \(\KendallTau(r_i^*,\mathcal{P}_{B_i(\MG_{\PP})})\!=\!\min_{r \in \RankSet(B_i(\MG_{\PP}))} \KendallTau(r,\mathcal{P}_{B_i(\MG_{\PP})})\). We have:
\begin{equation*}
    \KendallTau(r^*,\PP)\!=\!\min_{r \in \RankSet(C)} \KendallTau(r,\PP).
\end{equation*}
That is, \(r^*\) is a consensus ranking according to the Kemeny rule. 
Furthermore, if \(\TopSortSet{\MG_{\PP}} = \{<_{\mathcal{G}_{\PP}}\}\) and \(w_{\PP}(c,c')>0\) for all \(c \in B_i(\MG_{\PP})\) and \(c' \in B_j(\MG_{\PP})\) when \(i\!<_{\mathcal{G}_{\PP}}\! j\), then all consensus rankings are consistent with \(<_{\mathcal{G}_{\PP}}\).
\end{theorem}
This result does not hold anymore if one uses \(\EKendallTau{k}\) (with \(k\!\ge\!3\)) instead of \(\KendallTau\), as shown by the following example.

\begin{example}
Let us denote by $\PP$ the profile of Example~\ref{ex:Baldiga}. 
The pairwise majority digraph \(\MG_{\PP}\) has three SCCs \(B_1(\MG_{\PP})=\{c_2\}\), \(B_2(\MG_{\PP})=\{c_3\}\) and \(B_3(\MG_{\PP})=\{c_1\}\).  In this example, \(\TopSortSet{\MG_{\PP}} = \{ <_{\MG_{\PP}} \}\) where \(1 <_{\MG_{\PP}} 2 <_{\MG_{\PP}} 3\). The only ranking consistent with \(<_{\MG_{\PP}}\) is \(c_2 \succ c_3 \succ c_1\) while the only consensus ranking w.r.t. the 3-wise Kemeny rule is \(c_1 \succ c_2 \succ c_3\).
\end{example}

In order to adapt Theorem~\ref{thrm:charon} to the $k$-wise Kemeny rule, we now introduce the concept of \emph{\(k\)-wise majority digraph}. Let \(\SubsetSet{k}_{cc'}(S)\!\!=\!\! \{S'\!\in\!\SubsetSet{k}(S) \text{ s.t. } \{c,c'\}\!\subseteq\!S'\}\). If \(S\) is not specified, it is assumed to be \(C\).
Given a ranking \(r\), we denote by \(\SubsetSet{k}_{r}(S,c,c')\) the set \(\{S'\!\!\in\!\!\SubsetSet{k}_{cc'}(S) \text{ s.t. } \tp_{r}(S')\!\!=\!\!c\}\). Given a profile \(\PP\), we denote by \(\phi^{k}_{\PP}(S,c,c')\) the value \(\sum_{r\in \PP} |\SubsetSet{k}_r(S,c,c')|\) and by \(w^k_{\PP}(S,c,c')\) the difference \(\phi^{k}_{\PP}(S,c,c')\!-\!\phi^{k}_{\PP}(S,c',c)\). This definition implies that \(w^k_{\PP}(S,c',c)\!\!=\!\!-w^k_{\PP}(S,c,c')\). 
The value \(w^k_{\PP}(S,c,c')\) is the net agreement loss that would be incurred by swapping \(c\) and \(c'\) in a feasible solution \(r\) of \(k\)-KAP where \(\rk(c',r)\!\!=\!\!\rk(c,r)\!+\!1\) and \(S\!\!=\!\!\bel_{c'}(r)\!\cup\!\{c,c'\}\).
If $\max_{S\in \SubsetSet{m}_{cc'}}$ $w^k_{\PP}(S,c,c')\!\!\geq\!\! 0$ (resp. \(\min_{S\in \SubsetSet{m}_{cc'}} \!w^k_{\PP}(S,c,c')\!\!>\!\!0\)) then, in a consensus ranking \(r\) for \(\EKendallTau{k}\) where \(c\) and \(c'\) would be consecutive, it is possible (resp. necessary) that \(c \succ_r c'\).
\revised{\begin{definition}
The \(k\)-wise majority digraph associated to a profile \(\PP\) over a set \(C\) of candidates is the digraph \(\EMG{k}_{\PP}\!=\!(\mathcal{V}, \mathcal{A})\), where \(\mathcal{V}\!=\!C\) and \((c,c')\!\in\!\mathcal{A}\) iff:
\begin{equation*}
    \exists S\in \SubsetSet{m}_{cc'} \text{ s.t. } w^k_{\PP}(S,c,c') > 0.
\end{equation*}
In the weighted \(k\)-wise majority digraph, each edge \((c,c')\) is weighted by:
\begin{equation*}
    w^k_{\PP}(c,c') \coloneqq \max_{S\in \SubsetSet{m}_{cc'}}
    w^k_{\PP}(S,c,c').    
\end{equation*}
\end{definition}}

Note that, if \(k\ge 3\), we may obtain edges \((c,c')\) and \((c',c)\) both with strictly positive weights (which is impossible in the pairwise case). For instance, for the profile \(\PP\) of Example~\ref{ex:MG},  \(\revtwo{w^3_{\PP}}(c_3,c_4)\!=\!\revtwo{w^3_{\PP}}(\{c_2,c_3,c_4\},c_3,c_4)\!=\!1\) and \(\revtwo{w^3_{\PP}}(c_4,c_3)\!=\!\revtwo{w^3_{\PP}}(\{c_3,c_4,c_5\},c_4,c_3)\!=\!4\). 
For illustration, let us explain how $\revtwo{w^3_{\PP}}(\{c_3,c_4,c_5\},c_4,c_3)$ yields $4$. 
\revtwo{Table~\ref{tab:IllustrationDef5} summarizes the number of times $c_4$ and $c_3$ appear in top position of $\{c_3,c_4\}$ or $\{c_3,c_4,c_5\}$ for each ranking $r$ in $\mathcal{P}$. By summing over $r\!\in\!\mathcal{P}$: $\phi^3_{\PP}(\{c_3,c_4,c_5\},c_4,c_3)-\phi^3_{\PP}(\{c_3,c_4,c_5\},c_3,c_4)\!=\!(4\times 2 + 2 + 2)\!-\!(4\times 2)\!=\!4$.}

\begin{table}[hbtp]
\begin{center}
\revtwo{
\caption{\label{tab:IllustrationDef5}Number of times $c_4$ and $c_3$ appear in top position of $\{c_3,c_4\}$ or $\{c_3,c_4,c_5\}$ for each ranking $r$ in the profile $\mathcal{P}$ of Example~\ref{ex:MG}.}
}
\begin{tabular}{ccc}
\hline
$r$ & $|\SubsetSet{3}_r(\{c_3,c_4,c_5\},c_4,c_3)|$ &
$|\SubsetSet{3}_r(\{c_3,c_4,c_5\},c_3,c_4)|$\\
\hline
$c_1 \succ c_2 \succ c_4 \succ c_3 \succ c_5 \succ c_6 $ $(\times 4)$ & 2 & 0 \\
$c_1 \succ c_3 \succ c_2 \succ c_4 \succ c_5 \succ c_6 $ $(\times 4)$ & 0 & 2 \\
$c_6 \succ c_1 \succ c_2 \succ c_4 \succ c_3 \succ c_5 $ $(\times 1)$ & 2 & 0 \\
$c_6 \succ c_1 \succ c_4 \succ c_3 \succ c_2 \succ c_5 $ $(\times 1)$ & 2 & 0 \\
\hline
\end{tabular}
\end{center}
\end{table}

The obtained weighted digraph \(\EMG{3}_{\PP}\) is shown on the right of Figure~\ref{fig:extendedMajorityGraph} \HG{(an efficient manner to compute a set $S$ maximizing \(w^k_{\PP}(S,c,c')\) will be explained later on page~\pageref{explanation:graph})}. Besides, for any \(\PP\), \(\EMG{2}_{\PP}\) is the pairwise majority digraph as \(\SubsetSet{2}_{cc'}(S) =  \{\{c,c'\}\}\) \(\forall S\!\in\! \SubsetSet{m}_{cc'}\). Theorem~\ref{thrm:charon} adapts as follows for an arbitrary \(k\): 

\begin{theorem} \label{prop:SCCs}
Let \(\PP\) be a profile over \(C\) and assume that the SCCs of \(\EMG{k}_{\PP}\) are numbered according to a linear order \(<_{\EMG{k}_{\PP}} \in\! \TopSortSet{\EMG{k}_{\PP}}\). Among the rankings consistent with \(<_{\EMG{k}_{\PP}}\), 
there exists a consensus ranking w.r.t. the \(k\)-wise Kemeny rule. 
Besides, if $\TopSortSet{\EMG{k}_{\PP}}$$=$$ \{ <_{\EMG{k}_{\PP}}\} $ 
and \(\min_{S\in \SubsetSet{m}_{cc'}} w^k_{\PP}(S,c,c')\!>\!0\)\footnote{Or, equivalently, \(\max_{S\in \SubsetSet{m}_{cc'}} w^k_{\PP}(S,c',c)\!<\!0\).} for all \(c\!\in\! B_i(\EMG{k}_{\PP})\) and \(c' \!\in\!B_j(\EMG{k}_{\PP})\) when \(i\!<_{\EMG{k}_{\PP}}\!j\), then all consensus rankings are consistent with \(<_{\EMG{k}_{\PP}}\).
\end{theorem}

\begin{proof}
Assume that the SCCs of \(\EMG{k}_{\PP}\) are numbered according to a linear order  \(<_{\EMG{k}_{\PP}} \in\! \TopSortSet{\EMG{k}_{\PP}}\) and consider a ranking \(r\) which is not consistent with \(<_{\EMG{k}_{\PP}}\). 
Hence, there exists a pair \((c,c')\) such that \(c\) directly follows \(c'\) in \(r\) while \(c\!\in\! B_i(\EMG{k}_{\PP})\) and \(c'\!\in\! B_j(\EMG{k}_{\PP})\) with \(i\!<\!j\). Since \(i\!<\!j\), there is no arc from \(c'\) to \(c\) in \(\EMG{k}_{\PP}\) (i.e., \(\forall S\!\in\!\SubsetSet{m}_{cc'}, w^k_{\PP}(S,c,c')\!\geq\!0\)). Let \(S\) be the set composed of \(c'\) and all candidates placed after \(c'\) in \(r\), including \(c\). 
Then the ranking \(r^{c\leftrightarrow c'}\) obtained from \(r\) by exchanging the positions of \(c\) and \(c'\) verifies \(\EKendallTau{k}(r^{c\leftrightarrow c'},\PP)\!=\! \EKendallTau{k}(r,\PP)\!-\!w^k_{\PP}(S,c,c') \le \EKendallTau{k}(r,\PP)\). The repetition of this argument concludes the proof of the first claim. The second claim is proved similarly because, in this case, \(\EKendallTau{k}(r^{c\leftrightarrow c'},\PP)\!=\! \EKendallTau{k}(r,\PP)\!-\!w^k_{\PP}(S,c,c')\!<\! \EKendallTau{k}(r)\). 
\end{proof}

\begin{example}
The meta-graph of SCCs of \(\EMG{3}_{\PP}\) in Example \ref{ex:MG} is represented in Figure~\ref{Fig:SCCs}. The above result implies that there exists a consensus ranking among \(c_1 \succ c_2 \succ c_3 \succ c_4 \succ c_5 \succ c_6\), \(c_1 \succ c_2 \succ c_3 \succ c_4 \succ c_6 \succ c_5\), \(c_1 \succ c_2 \succ c_4 \succ c_3 \succ c_5 \succ c_6\) and \(c_1 \succ c_2 \succ c_4 \succ c_3 \succ c_6 \succ c_5\).
\end{example}

\begin{figure}[t] 
   \centering
    \scalebox{1}{\begin{tikzpicture}[->,>=stealth',shorten >=1pt,auto,node distance=5cm,semithick]

  \node[draw=none,text=black] (B1)   at (0,0)                 {\large $B_1 = \{c_1\}$};
    \node[draw=none,text=black] (B2)   at (4,0)                 {\large $B_2 = \{c_2\}$};
  \node[draw=none,text=black] (B3)   at (8,0)                 {\large $B_3 = \{c_3,c_4\}$};
  \node[draw=none,text=black] (B4)   at (12,0)                 {\large $B_4 = \{c_5,c_6\}$};

  \path (B1) edge  node {} (B2)
        (B1) edge[bend right=15]  node {} (B3)
        (B1) edge[bend right=10]  node {} (B4)
        
        (B2) edge  node {} (B3)
        (B2) edge[bend left=15]  node {} (B4)
    
        (B3) edge  node {} (B4);
\end{tikzpicture}}
\caption{The meta-graph of SCCs of \(\EMG{3}_{\PP}\) in Example \ref{ex:MG}.}
    \label{Fig:SCCs}
\end{figure}

To take advantage of  Theorem~\ref{prop:SCCs}, one could try 1) to index the SCCs of \(\EMG{k}_{\PP}\) according to a linear order \(<_{\EMG{k}_{\PP}} \in \TopSortSet{\EMG{k}_{\PP}}\), and then 2) to work on each SCC separately, before concatenating the obtained rankings. 
However, for a consensus ranking consistent with \(<_{\EMG{k}_{\PP}}\), the relative positions of candidates in \(B_i(\EMG{k}_{\PP})\) depend on the set of candidates in \(B_{> i}(\EMG{k}_{\PP})\!\coloneqq\! B_{i+1}(\EMG{k}_{\PP}) \cup \ldots \cup B_{\sigma(\EMG{k}_{\PP})}(\EMG{k}_{\PP})\) (but not on their order). The influence of \(B_{> i}(\EMG{k}_{\PP})\) can be captured in the dynamic programming procedure by applying a modified version of Equation~\ref{eq:progdyn} separately for each subset \(B_{\sigma(\EMG{k}_{\PP})}(\EMG{k}_{\PP})\) downto
\(B_1(\EMG{k}_{\PP})\).
Formally, if \(r^*\) is optimal for \(k\)-KAP, then: 
\begin{equation*}
\EKendallTau{k}(r^*,\PP)\!=\!\sum_{i=1}^{\sigma(\EMG{k}_{\PP})} d_{\mathtt{KT}}^k(B_i(\EMG{k}_{\PP}))
\end{equation*}
where, for any subset \(S \subseteq B_i(\EMG{k}_{\PP})\), \(d_{\mathtt{KT}}^k(S)\) is defined by $d_{\mathtt{KT}}^k(\emptyset) = 0$ and (\(B_{> i}\) stands for \(B_{> i}(\EMG{k}_{\PP})\)): 
\begin{align*}
    d_{\mathtt{KT}}^k(S)& = \min_{c\in S}  [d_{\mathtt{KT}}^k(S\!\setminus \!\{c\})\\ & +\!\!\!\sum_{r \in \PPr{S\cup B_{> i}}} \sum_{c'\succ_{r} c} \sum_{i=0}^{k-2}  \binom{|S|+|B_{> i}|-\rk(c',r)-1}{i}].
\end{align*}
It amounts to replacing \(S\) by \(S\!\cup\! B_{>i}\) in the second summand of Equation~\ref{eq:progdyn} to take into account the existence of a consensus ranking where all the candidates of \(B_{>i}\) are ranked after those of \(B_{i}\). Let \(r^*_i\) be a ranking of \(B_i(\EMG{k}_{\PP})\) such that \(\EKendallTau{k}(r_{\ge i}^*,\PP_{B_{\ge i}(\EMG{k}_{\PP})})\!=\! d_{\mathtt{KT}}^k(B_i(\EMG{k}_{\PP}))\!+\!\ldots\!+\!d_{\mathtt{KT}}^k(B_{\sigma(\EMG{k})}(\EMG{k}_{\PP}))\), where \(r_{\ge i}^*\) is the ranking obtained by the concatenation of rankings \(r^*_i,\ldots,r^*_{\sigma(\EMG{k})}\) in this order. The ranking \(r^*_{\ge 1}\) of \(C\) is a consensus ranking w.r.t. the $k$-wise Kemeny rule. 
Given Theorem \ref{prop:SCCs}, the \(k\)-wise majority digraph thus seems promising to boost the computation of a consensus ranking. Unfortunately, the following negative result holds: 

\HG{
\begin{theorem} \label{thrm:MajorityHardness}
Given two candidates \(c\) and \(c'\) in a profile \(\PP\), determining if \(\displaystyle\max_{S\in \SubsetSet{m}_{cc'}} w^k_{\PP}(S,c,c') > 0\) is an NP-complete problem for any constant \(k\!\geq\!4\). 
\end{theorem}
}
\begin{proof}
\HG{For the membership part, note that given a set $S$, we can check in polynomial time if \(w^k_{\PP}(S,c,c') > 0\). Indeed, we can enumerate all sets in $\SubsetSet{k}_{cc'}(S)$ and for each ranking $ r \in \PP$ count how many are in $\SubsetSet{k}_{r}(S,c,c')$ or in  $\SubsetSet{k}_{r}(S,c',c)$.}

\HG{For the hardness part,} we make a reduction from the set cover problem, known to be NP-complete~\cite{karp1972reducibility}:\\ [1.5ex]
{\sc Set Cover Problem}\\
\emph{Instance:} A set of elements \(\mathcal{X} = \{x_1, \ldots, x_p\}\), a collection \(\mathcal{T} = \{T_1,\ldots, T_q\}\) of sets of elements of \(\mathcal{X}\), \HG{and a positive integer $b$}.\\
\emph{Question:} Does there exist a subcollection \(\mathcal{K}\subseteq \mathcal{T}\) of at most \(b\) sets that covers \(\mathcal{X}\) (i.e., such that \(\bigcup_{T \in \mathcal{K}} T\!=\!\mathcal{X}\))? \\ [1.5ex]
We assume that no set in \( \mathcal{T} \) contains \(\mathcal{X}\) as otherwise the problem is trivial. Furthermore, we assume that no element in \(\mathcal{X}\) is contained in all sets of \( \mathcal{T} \) as otherwise this element could be discarded from the instance as any solution would cover this element. We now detail the preference profile that we create from an instance of the set cover problem.

\medskip

\textit{Set of candidates:} We will create a profile such that \(\max_{S\in \SubsetSet{m}_{cc'}} w^k_{\PP}(S,c,c')>0\) iff the answer to the set cover problem is yes for the instance under consideration. More precisely, we will show that if \(\mathcal{X}\) cannot be covered by a subcollection \(\mathcal{K}\subseteq \mathcal{T}\) with less than \(b\) sets, then \(w^k_{\mathcal{P}}(S,c,c') < 0\) for all \(S\) in \(\SubsetSet{m}_{cc'}\). 
Otherwise, if set \(\mathcal{X}\) can be covered with a subcollection \(\mathcal{K}\subseteq \mathcal{T}\) with less than \(b\) sets, then there exists a set \(S\) in \(\SubsetSet{m}_{cc'}\) such that \(w^k_{\mathcal{P}}(S,c,c') > 0\). 
In addition to candidates \(c\) and \(c'\), for each pair \HG{\((x,T) \in \mathcal{X}\times \mathcal{T} \)} such that \HG{\( x \in T \)} we create a candidate \HG{\(c_{x,T}\)}. Moreover, for each set \HG{\(T \in \mathcal{T} \)} we create a candidate \HG{\(c_T\)}. 
In the sequel, we may call candidates \HG{\(c_{x,T}\)} \emph{element candidates} and candidates \HG{\(c_T\)} \emph{set candidates}. This process yields at most \(pq + 2\) candidates (\((p-1)q + q + 2 \)). 
The candidates \HG{\(c_{x,T}\)} and \HG{\(c_{T}\)} will make the correspondence with the subcollection \(\mathcal{K}\): the candidate \HG{\(c_{T}\)} will be in the set \(S\) iff \(\HG{T} \in \mathcal{K}\)  and the candidate \HG{\(c_{x,T}\)} will be in the set \(S\) if \HG{\(T\)} is added to \(\mathcal{K}\) in order to cover \HG{\(x\)}.

\medskip

\textit{Set of voters:} For each pair \HG{\( (x,T) \in \mathcal{X}\times \mathcal{T} \)} such that \HG{\( x \in T \)}, we create \(2b\) voters \HG{\(v_{x,T}^s\)} (\( s \in \{1,\ldots, 2b\} \)) with the same ranking \HG{\(r_{x,T}\)} such that \HG{\( \mathtt{tail}_4(r_{x,T}) =  c \succ c_{x,T} \succ c_T \succ c'\).}  
For each element \HG{\(x \in \mathcal{X}\)}, we create \(2b\) voters  \HG{\( v_{x}^s \)} (\(s\in \{1,\ldots, 2b\}\))  with the same ranking \HG{\(r_{x}\)} such that \( \HG{\mathtt{tail}_{u+2}(r_{x}) = c' \succ c_{x, T_{i_1}} \succ \ldots \succ c_{x, T_{i_u}} \succ c }\) where \HG{\(T_{i_1}, T_{i_2}, \ldots, T_{i_u}\)} are the different sets in \HG{$\mathcal{T}$} that contain \(\HG{x}\).   
For each set \HG{\(T\in \mathcal{T}\)} we create \HG{\(2b|T|+2\)} voters \( \HG{v_{T}^s}\) (\(s\in \{1,\ldots, \HG{2b|T|+2}\}\)) with the same ranking \(\HG{r_{T}}\) such that \HG{\( \mathtt{tail}_{3}(r_{T}) = c' \succ c_{T} \succ c\)}. 
Lastly, we create \(2q+1+2b\) voters \(v^s\) (\(s\in \{1,\ldots, 2q+1+2b\}\)) with the same ranking \(r\)  such that \( \mathtt{tail}_{2}(r) = c \succ c' \). In the end, we obtain \(O(bpq)\) voters. Note that \(\mathcal{P}\) can be build with no unanimity dominance relationship between two candidates.

\medskip

We now show that, given \(k\ge 4\), we have \(\max_{S\in \SubsetSet{m}_{cc'}} w^k_{\PP}(S,c,c') > 0\) 
iff the answer to the set cover problem is yes.

\medskip

First note that when \(S = \{c, c'\}\), we have the following values for \(\phi^k_{\mathcal{P}}(S,c,c')\) and \(\phi^k_{\mathcal{P}}(S, c', c)\):
\begin{align*}
\phi^k_{\mathcal{P}}(\{c, c'\},c,c') &= \HG{2b \sum_{T\in \mathcal{T}} |T| + 2q + 2b +1}\\
\phi^k_{\mathcal{P}}(\{c, c'\},c',c) &= \HG{2b \sum_{T\in \mathcal{T}} |T| + 2q + 2bp}.
\end{align*}
Hence, 
\begin{align*}
w^k_{\mathcal{P}}(\{c, c'\},c,c') &=  \phi^k_{\mathcal{P}}(\{c, c'\},c,c') - \phi^k_{\mathcal{P}}(\{c, c'\},c',c)\\
&= 2b(1-p) + 1 < 0,    
\end{align*}
as \(p\ge 2\) and \(b\ge 1\). 

Now let us look at how adding an element to \(S\) starting from \(S = \{c,c'\}\) modifies values \(\phi^k_{\mathcal{P}}(S,c,c')\) and \(\phi^k_{\mathcal{P}}(S,c',c)\).  
We will assume without loss of generality that we add set candidates before element candidates. 
\begin{itemize}
\item If we assume \(S\) is only composed of \(c\), \(c'\) and set candidates, then adding a candidate \HG{\(c_T\)} to it results in adding \HG{\(2b|T|\)} to \(\phi^k_{\mathcal{P}}(S,c,c')\) and \HG{\(2b|T| + 2\)} to \(\phi^k_{\mathcal{P}}(S,c',c)\).
\item If we add a candidate \HG{\(c_{x,T}\)} to \(S\), then we add \(2b\) to \(\phi^k_{\mathcal{P}}(S,c,c')\) if \(\HG{c_T}\not\in S\) and \(4b\) otherwise. Additionally, we add \(2b\) to \(\phi^k_{\mathcal{P}}(S,c',c)\) if there is no other \HG{\(c_{x,T'} \in S\)}, otherwise we add to it something that is greater than or equal to \(4b\). Note that we have used the fact that \(k\ge 4\), as we would only add \(2b\) to both values if \(k\) was equal to 3.
\end{itemize}
From these observations, we can derive the following rules:
\begin{enumerate}
\item If \HG{\(\{c_{x,T},c_{x,T'}\} \subset S\)} with \HG{\(T \neq T'\)}, then \(w^k_{\mathcal{P}}(S,c,c') \le \HG{w^k_{\mathcal{P}}(S\setminus\{c_{x,T'}\},c,c')}\). 
\item If \(\HG{c_{x,T}\in S}\) and \(\HG{c_T\not \in S}\), then 
\(w^k_{\mathcal{P}}(S,c,c') \le  w^k_{\mathcal{P}}(S\setminus\{\HG{c_{x,T}}\},c,c')\). 
\item If \(\HG{c_{T}}\in S\) and \(\HG{\forall x\in T}, c_{x,T}\not \in S\), then \(w^k_{\mathcal{P}}(S,c,c')\) \(\le  w^k_{\mathcal{P}}(S\setminus\{\HG{c_{T}}\},c,c')\).  
\end{enumerate}
Stated differently, while rule 1 states that, to maximize \(w^k_{\mathcal{P}}(S,c,c')\), we should keep no more than one candidate \(\HG{c_{x,T}}\) per element \(\HG{x}\), rules 2 and 3 state that there should be a candidate \(\HG{c_{x,T}}\) in \(S\) iff there should also be candidate \(\HG{c_T}\).\\

Let us now consider a set \(S\in \SubsetSet{m}_{cc'}\) verifying rules 1, 2 and 3 that includes one candidate \HG{\(c_{x,T}\)} for \(s\) elements \(\{x_{i_1},\ldots,\) \( x_{i_s}\}\) and \(v\) set candidates. Then:
\begin{align*}
w^k_{\mathcal{P}}(S,c,c') &=  2b(1-p) + 1 - 2v + 2bs\\
                                 &=  2b(s-p) + 2(b-v) +1. 
\end{align*} 
Note that if \(S\neq \{c,c'\}\), then \(v\ge 1\) and that by rule 1, \(s\le p\). If \(s < p\), then \( w^k_{\mathcal{P}}(S,c,c') \le  -2b + 2(b-v) +1 = 1- 2v < 0 \). Hence, if \(w^k_{\mathcal{P}}(S,c,c')\ge 0\) then \(s = p\), which further implies that \(v \le b\). In this case, it is easy to see that the \(v\) sets corresponding to the set candidates in \(S\) form a valid set cover of \(\mathcal{X}\) as \(\{x_{i_1},\ldots, x_{i_s}\} = \mathcal{X}\), \(v \le b\) and \(S\) verifies rule 2. To summarize, making the assumption that \(w^k_{\mathcal{P}}(S,c,c')\ge 0\), (\HG{In fact in this case \(w^k_{\mathcal{P}}(S,c,c') > 0\)}) we have showed that we could build a valid set cover of \(\mathcal{X}\) from \(S\). Consequently, this implies that if the set cover instance admits no valid set cover, then  \(\max_{S\in \SubsetSet{m}_{cc'}} w^k_{\PP}(S,c,c') < 0\).\\

Conversely, let us assume that there exists a subcollection \(\HG{\mathcal{K}} = \{ T_{i_1},\) \(\ldots , T_{i_v} \}\) with \(v \le b\) sets that covers \(\mathcal{X}\). We consider a set \(S\) such that \HG{\(\{c_{T_{i_1}},\ldots,c_{T_{i_v}},c,c'\} \subset S\)} and such that for each \HG{\(x \in\mathcal{X}\)}, \(S\) contains exactly one candidate \HG{\(c_{x,T}\)} where \HG{\(x \in T\)} and \(T \in \mathcal{K}\). Then, simple calcula show that:
 \begin{align*}
 \phi^k_{\mathcal{P}}(S,c,c') &= 2b\HG{\sum_{T\in \mathcal{T}} |T|} + 2b\HG{\sum_{T\in \mathcal{K}} |T|} + 4bp + 2q + 1 + 2b,\\  
 \phi^k_{\mathcal{P}}(S,c',c) &= 4bp + 2b\HG{\sum_{T\in \mathcal{T}} |T|} + 2q + 2b\HG{\sum_{T\in \mathcal{K}} |T|} + 2v  .
\end{align*}
Hence, \(w^k_{\mathcal{P}}(S,c,c') \ge 1\) and \(\max_{S\in \SubsetSet{m}_{cc'}} w^k_{\PP}(S,c,c') > 0\).  
\end{proof}

Hence, computing \(\EMG{k}_{\PP}\) from \(\PP\) is NP-hard for \(k\!\geq\!4\). In contrast, \(\EMG{3}_{\PP}\) can be computed in polynomial time. Indeed, given a set \(S\!\subset\!C\) such that \(\{c,c'\}\!\subseteq\!S\), adding an element \(x\!\not\in\!S\) to \(S\) increases \(\phi^3_{\PP}(S,c,c')\) by one for each \(r\!\in \!\PP\) such that \(c\!\succ_{r}\!c'\) and \(c\!\succ_{r}\!x\). Let us denote by \(\PP_{c\succ c'}\) the set \(\{r\!\in\!\PP \mbox{ s.t. } c\!\succ_{r}\!c'\}\). A set $S^*$ maximizing \(w^3_{\PP}(S,c,c')\) is  \(S^*\!\coloneqq\!\{c,c'\}\cup\{x\in C \mbox{ such that } |\PP_{c\succ c'}\cap\PP_{c\succ x}|\!>\!|\PP_{c'\succ c}\cap\PP_{c'\succ x}|\}\).
\HG{\label{explanation:graph}
For instance, coming back to the weighted digraph \(\EMG{3}_{\PP}\) shown on the right of Figure~\ref{fig:extendedMajorityGraph}, a set \(S^*\) maximizing \(w^3_{\PP}(S,c_4,c_3)\) in Example~\ref{ex:MG}  
is \(\{c_3,c_4,c_5\}\), which yields \(w^3_{\PP}(S^*,c_4,c_3) = 4\) as reported previously.}\\

Note that one can take advantage of the meta-graph of SCCs to trim the graph \(\EMG{k}_{\PP}\) if one looks for a consensus ranking \(r^*\) consistent with a specific order \(<_{\EMG{k}_{\PP}} \in \TopSortSet{\EMG{k}_{\PP}}\). 
It may indeed happen that, for an edge \((c,c')\), the weight \(w^k_{\PP}(c,c')\!=\!w^k_{\PP}(S,c,c')\!>\!0\) corresponds to a set \(S\) which contains candidates that will never be below \(c\) in \(r^*\). Conversely, the set \(S\) may omit candidates that are necessarily below \(c\) in \(r^*\). These constraints can be induced by either unanimity dominance relations or by  \(<_{\EMG{k}_{\PP}}\). The following example illustrates this idea. 

\begin{example}
Let us refine the digraph \(\EMG{3}_{\PP}\) previously obtained for the profile \(\PP\) of Example \ref{ex:MG}. The SCCs are \(B_1 = \{c_1\}\), \(B_2 = \{c_2\}\), \(B_3 = \{c_3, c_4\}\) and \(B_4 = \{c_5,c_6\}\) and \(\TopSortSet{\EMG{k}_{\PP}} = \{<_{\EMG{k}_{\PP}}\}\), where \(1 <_{\EMG{k}_{\PP}} 2 <_{\EMG{k}_{\PP}} 3 <_{\EMG{k}_{\PP}} 4\). 
A set maximizing \(w^3_{\PP}(S,c_3,c_4)\) is \(S\!=\!\{c_2,c_3,c_4\}\). This set contains \(c_2\) while it is necessarily above \(c_3\) in a consistent ranking. Conversely, candidates \(c_5\) and \(c_6\) are omitted while they are necessarily below \(c_3\). By taking into account these constraints, we obtain that a set maximizing \(w^3_{\PP}(S,c_3,c_4)\) is  \(S\!=\!\{c_3,c_4,c_5,c_6\}\), for which \(w^3_{\PP}(S,c_3,c_4)\!=\!-4\). Hence, we can remove the arc \((c_3,c_4)\) from \(\EMG{3}_{\PP}\). Similarly, it is  possible to show that the arc \((c_6,c_5)\) can be removed from \(\EMG{3}_{\PP}\). Thanks to these refinement steps, we can conclude that a consensus ranking is \(r^* = c_1 \succ c_2 \succ c_4 \succ c_3 \succ c_5 \succ c_6\).
\end{example}

\color{black}
\section{A polynomial time 2-approximation algorithm}\label{sec:approx}

In this section, we provide a polynomial time 2-approximation algorithm for problem \(k\)-KAP, in the same spirit as the 2-approximation algorithm by \citeauthor{dwork2001rank} \cite{dwork2001rank} for the Kemeny aggregation problem. 
\HG{
This latter algorithm relies on the Spearman distance.
\begin{definition}
The Spearman distance $\delta_S(r,r')$ between two rankings $r$ and $r'$ is defined by 
$$\delta_S(r,r') =  \sum_{c \in C} |\rk(c,r)-\rk(c,r')|.$$
\end{definition}
}
The \emph{Diaconis-Graham inequality} \cite{diaconis1977spearman} \HG{states that the Kendall tau and Spearman distances} remain within a constant factor of each other
for all pairs of rankings, namely:
\[
\KendallTau(r,r') \le \Spearman(r,r') \le 2\KendallTau(r,r').
\]
Note that the right bound is tight, as witnessed by the two rankings $c_1 \succ_r c_2$ and $c_2 \succ_{r'} c_1$, for which \(\KendallTau(r,r')\!=\!1\) and \(\Spearman(r,r')\!=\!2\).

\citeauthor{dwork2001rank} \cite{dwork2001rank} have shown that a ranking \(r_S\) minimizing the sum of Spearman distances to the rankings in a preference profile \(\PP\) (i.e.,  $\sum_{r' \in \PP} \Spearman(r_S,r')\!=\!\min_r \sum_{r' \in \PP} \Spearman(r,r')$) can be computed in polynomial time via a minimum cost matching algorithm. As a consequence of the Diaconis-Graham inequality, the sum of Kendall tau distances between \(r_S\) and the rankings in \(\PP\) is in the worst case twice that of an optimal ranking for the Kemeny rule.
Denoting by $r_{\mathtt{KT}}$ an optimal ranking for the Kemeny rule on $\PP$, we have indeed:
\[
\sum_{r' \in \PP} \KendallTau(r_{\mathtt{S}},r') \le \sum_{r' \in \PP} \Spearman(r_{\mathtt{S}},r') \le \sum_{r' \in \PP} \Spearman(r_{\mathtt{KT}},r') \le 2\sum_{r' \in \PP} \KendallTau(r_{\mathtt{KT}},r').
\]

In order to obtain the same kind of result for the $k$-wise Kendall tau distance, we introduce \HG{a $k$-wise variant of the Spearman distance.
\begin{definition}
The $k$-wise Spearman distance $\ESpearman{k}(r,r')$ between two rankings $r$ and $r'$ is defined by:
\[
\ESpearman{k}(r,r')=\sum_{c\in C} \textstyle\sum_{i=\min\{\rk(c,r),\rk(c,r')\}}^{\max\{\rk(c,r),\rk(c,r')\}-1}N^{m-i-1}_{k-2}
\]
where $N^p_k\!=\!\sum_{i=0}^{k} \binom{p}{i}$ is the number of subsets of size less than or equal to $k$ \HG{within} a set of size $p$. 
\end{definition}
This definition is motivated by the fact that} $N^{m-i-1}_{k-2}$ corresponds to the $k$-wise Kendall tau distances between a ranking and the ranking obtained by swapping candidates of ranks $i$ and $i+1$. Note that $\ESpearman{2}\!=\!\Spearman$. We have indeed: \[
\begin{array}{rl}
\sum_{i=\min\{\rk(c,r),\rk(c,r')\}}^{\max\{\rk(c,r),\rk(c,r')\}-1}N^{m-i-1}_0
= & \max\{\rk(c,r),\rk(c,r')\}-\min\{\rk(c,r),\rk(c,r')\}\\
= & |\rk(c,r)\!-\!\rk(c,r')|.
\end{array}
\]
because $N^{m-i-1}_0\!=\!1$ for all $i$ in the summation. 

\HG{We now state and prove the extension of the Diaconis-Graham inequality to our $k$-wise variants: }
\begin{lemma} \label{lem:EDiaconisGraham}
For any two rankings $r$ and $r'$: $\EKendallTau{k}(r,r')\!\le\! \ESpearman{k}(r,r')\!\le\! 2\EKendallTau{k}(r,r')$.
\end{lemma}

\begin{proof}
We first prove that $\EKendallTau{k}(r,r')\!\le\! \ESpearman{k}(r,r')$. We recall that:
\[ 
\EKendallTau{k}(r,r')=|\{S \in \SubsetSet{k}:\tp_r(S)\neq \tp_{r'}(S)\}|
\]
which can be reformulated:
\[
\EKendallTau{k}(r,r')=\sum_{c \in C} |\{S \in \SubsetSet{k}:\tp_r(S) = c \mbox{ and } \tp_{r'}(S)\neq c\}|.
\]
Let us denote by $\SubsetSet{k}_c(r,r')$ the set $\{S \in \SubsetSet{k}:\tp_r(S)=c \mbox{ and } \tp_{r'}(S)\neq c\}$. For any $S\!\in\!\SubsetSet{k}_c(r,r')$, one of the followings holds true:
\begin{itemize}
\setlength\itemsep{0.1em}
    \item $\rk(\tp_r(S),r)\!\le\! \rk(\tp_{r'}(S),r')\!<\! \rk(\tp_r(S),r')$,
    \item or $\rk(\tp_{r'}(S),r)\!>\! \rk(\tp_{r}(S),r)\!\ge\! \rk(\tp_{r'}(S),r')$,
    \item or both  (if $\rk(\tp_r(S),r)\!=\!\rk(\tp_{r'}(S),r')$).
\end{itemize}
\[
\begin{aligned}
\mbox{Consequently: } \EKendallTau{k}(r,r')&\le
\displaystyle\sum_{c \in C} |\{S \!\in\!\SubsetSet{k}_c(r,r'):\rk(c,r)\!\le\! \rk(\tp_{r'}(S),r')\!<\! \rk(c,r')\}|\\
&+\displaystyle\sum_{c \in C} |\{S \!\in\!\SubsetSet{k}_c(r,r'):\rk(\tp_{r'}(S),r)\!>\! \rk(c,r)\!\ge\! \rk(\tp_{r'}(S),r')\}|
\end{aligned}
\]
because the union of the two sets is $\SubsetSet{k}_c(r,r')$ and $|A\cup B|\le |A|+|B|$ for any two sets $A$ and $B$.

On the one hand, we have:
\[
\begin{aligned}
& \displaystyle\sum_{c \in C} |\{S \!\in\!\SubsetSet{k}_c(r,r'):\rk(c,r)\!\le\! \rk(\tp_{r'}(S),r')\!<\! \rk(c,r')\}|\\ \le & \displaystyle\sum_{\substack{c \in C \\ \rk(c,r) < \rk(c,r')}} \textstyle\sum_{i=\rk(c,r)}^{\rk(c,r')-1}N_{k-2}^{m-i-1}.
\end{aligned}
\]
The upper bound results from the fact that each candidate $c'$ such that $\rk(c,r)\!\le\! \rk(c',r')\!<\! \rk(c,r')$ belongs to at most $N_{k-2}^{m-\rk(c',r')-1}$ subsets $S\!\in\! \SubsetSet{k}_c(r,r')$ such that $\tp_{r'}(S)\!=\!c'$.

On the other hand, we have:
\[
\begin{aligned}
& \displaystyle\sum_{c \in C} |\{S \!\in\!\SubsetSet{k}_c(r,r'):\rk(\tp_{r'}(S),r)\!>\! \rk(c,r)\!\ge\! \rk(\tp_{r'}(S),r')\}| \\
= & \displaystyle\sum_{c' \in C} |\{S \!\in\!\SubsetSet{k}_{c'}(r',r):\rk(c',r)\!>\! \rk(\tp_r(S),r)\!\ge\! \rk(c',r')\}|\\
\le & \displaystyle\sum_{\substack{c' \in C \\ \rk(c',r) > \rk(c',r')}} \textstyle\sum_{i=\rk(c',r')}^{\rk(c',r)-1}N_{k-2}^{m-i-1}\!=\!\displaystyle\sum_{\substack{c \in C \\ \rk(c,r) > \rk(c,r')}} \textstyle\sum_{i=\rk(c,r')}^{\rk(c,r)-1}N_{k-2}^{m-i-1}.
\end{aligned}
\]
The first equality results from the fact that varying $c$ and considering subsets $S$ such that $\tp_{r}(S)\!=\!c$ and $\rk(\tp_{r'}(S),r)>\rk(c,r)\ge\rk(\tp_{r'}(S),r')$ is equivalent to varying $c'$ and considering subsets $S$ such that $\tp_{r'}(S)\!=\!c'$ and $\rk(c',r)>\rk(\tp_{r}(S),r)\ge\rk(c',r')$.

Hence:
\[
\begin{aligned}
\EKendallTau{k}(r,r') & \le
\displaystyle\sum_{\substack{c \in C \\ \rk(c,r) < \rk(c,r')}} \textstyle\sum_{i=\rk(c,r)}^{\rk(c,r')-1}N_{k-2}^{m-i-1}+\displaystyle\sum_{\substack{c \in C \\ \rk(c,r) > \rk(c,r')}} \textstyle\sum_{i=\rk(c,r')}^{\rk(c,r)-1}N_{k-2}^{m-i-1}\\
& = \sum_{c\in C} \textstyle\sum_{i=\min\{\rk(c,r),\rk(c,r')\}}^{\max\{\rk(c,r),\rk(c,r')\}-1}N^{m-i-1}_{k-2} = \ESpearman{k}(r,r').
\end{aligned}
\]

\medskip

We now prove that $\ESpearman{k}(r,r')\!\le\! 2\EKendallTau{k}(r,r')$. In this purpose, we consider a sequence $r_0\!=\!r,r_1,\ldots,r_\delta\!=\!r'$ of rankings, where $r_{j+1}$ is obtained from $r_{j}$ by swapping in $r_j$ the candidate $c_j'\!=\!\tp_{r'}(\{c \in C: \rk(c,r_j)\neq\rk(c,r')\})$ with the previous candidate in $r_j$, denoted by $c_j$. 
This is like doing an in-place selection sort w.r.t. the order defined by $r'$, moving $c_j'$ to its place in $r'$ by successive swaps. 
Note the following things: 
\begin{itemize}
    \item at each step $\rk(c_j',r')\!<\! \rk(c_j',r_j)$ by definition of $c_j'$;
    \item $\delta\!=\!\KendallTau(r,r')$ because each swap of consecutive candidates in $r_j$ decreases by exactly one the number of pairwise disagreements between $r_j$ and $r'$ ($c_j'\succ_{r'} c_j$ by definition of $c_j'$).
\end{itemize}

The proof is in two steps:
\begin{enumerate}
    \item[$(i)$] We show that $\EKendallTau{k}(r_j,r')=\EKendallTau{k}(r_j,r_{j+1})+\EKendallTau{k}(r_{j+1},r')$.  By induction, an immediate consequence is that: 
    \[
    \EKendallTau{k}(r,r')=\sum_{j=0}^{\delta-1} \EKendallTau{k}(r_j,r_{j+1}).
    \]
    \item[$(ii)$] We define $D_j\!=\!\ESpearman{k}(r_j,r')$. We show that $D_j\!-\!D_{j+1} \!\le\! 2\EKendallTau{k}(r_j,r_{j+1})$, which implies: $\ESpearman{k}(r,r')\!=\!\sum_{j=0}^{\delta-1} (D_j-D_{j+1}) \!\le\!2\sum_{j=0}^{\delta-1} \EKendallTau{k}(r_j,r_{j+1})\!=\!2\EKendallTau{k}(r,r').$
\end{enumerate}    
    
    \noindent \emph{Proof of} $(i)$. 
    We have the following sequence of equalities:
    \[
    \begin{aligned}
    &\EKendallTau{k}(r_j,r')-\EKendallTau{k}(r_{j+1},r')\\
    &=|\{S \in \SubsetSet{k}:\tp_{r_j}(S)\neq \tp_{r'}(S)\}|-|\{S \in \SubsetSet{k}:\tp_{r_{j+1}}(S)\neq \tp_{r'}(S)\}|\\
    &=|\{S \in \SubsetSet{k}:\tp_{r_j}(S)=c_j \mbox{ and } \tp_{r'}(S)=c_{j}'\}|\\
    &\mbox{(because $r_j$ and $r_{j+1}$ only differ in the order of $c_j$ and $c_j'$)}\\
    &=|\{S \in \SubsetSet{k}:\{c_j,c_j'\}\subseteq S \mbox{ and } S\setminus\{c_j,c_j'\}\subseteq \bel_{c_j'}(r_j)\}|\\
    &\mbox{(because $c_j'\succ_{r'} c$, $\forall c \in \bel_{c_j'}(r_j)$, which implies that $\bel_{c_j}(r_j)\subseteq \bel_{c_j'}(r')$)}\\
    &=|\{S \in \SubsetSet{k}:\tp_{r_j}(S)\neq \tp_{r_{j+1}}(S)\}|\\
    &\mbox{(because $r_j$ and $r_{j+1}$ only differ in the order of $c_j$ and $c_j'$)}\\
    &=\EKendallTau{k}(r_j,r_{j+1}).
    \end{aligned}
    \]
    The result follows.
    
    \medskip
    
    \noindent \emph{Proof of} $(ii)$. 
    As $r_j$ and $r_{j+1}$ only differ in the ranks of $c_j$ and $c_j'$, the difference $D_j-D_{j+1}$ is equal to:
    \[
    \sum_{c\in\{c_j,c_j'\}}\left(\textstyle\sum_{i=\min\{\rk(c,r_j),\rk(c,r')\}}^{\max\{\rk(c,r_j),\rk(c,r')\}-1} N_{k-2}^{m-i-1} - \textstyle\sum_{i=\min\{\rk(c,r_{j+1}),\rk(c,r')\}}^{\max\{\rk(c,r_{j+1}),\rk(c,r')\}-1}N_{k-2}^{m-i-1}\right).
    \]
    
    Note that $\rk(c_j,r_{j+1}) = \rk(c_j,r_{j})+1$ and $\rk(c_j',r_{j+1}) = \rk(c_j',r_{j})-1$ and $\rk(c_j',r') \le \rk(c_j',r_{j+1})$. 
    By swapping $c_j$ and $c_j'$ in $r_j$, the $k$-wise Spearman distance to $r'$ will thus decrease with respect to $c_j'$ by $N_{k-2}^{m-\rk(c_j',r_{j})}$. Regarding $c_j$, it may increase by $N_{k-2}^{m-\rk(c_j,r_{j})-1}$ if $\max\{\rk(c_j,r_j),\rk(c_j,r')\} \!=\!\rk(c_j,r_j)$, or decrease by $N_{k-2}^{m-\rk(c_j,r_{j})-1}$ otherwise. 
    Hence, the largest possible decrease in $\ESpearman{k}$ occurs in this latter case. 
    We derive from this short analysis:
    \begin{align*}
        D_j - D_{j+1} &\le N_{k-2}^{m-\rk(c_j',r_{j})} + N_{k-2}^{m-\rk(c_j,r_{j})-1} \\
                      &= 2 N_{k-2}^{m-\rk(c_j',r_{j})} \text{ (as $\rk(c_j',r_{j}) = \rk(c_j,r_{j})+1$)}\\
                      &= 2 \EKendallTau{k}(r_j,r_{j+1})
    \end{align*}
    The last line results from the fact that, as stated previously, $N^{m-i-1}_{k-2}$ corresponds to the $k$-wise Kendall tau distance between a ranking and the ranking obtained by swapping candidates of ranks $i$ and $i+1$. 
\end{proof}

The following result states that a consensus ranking for the \(k\)-wise Spearman distance can be computed in polynomial time in the numbers \(n\) of voters and \(m\) of candidates:

\begin{lemma} \label{lem:Spearman}
A consensus ranking for the \(k\)-wise Spearman distance can be computed in time \(O(nm^3)\), for a preference profile with \(n\) voters and \(m\) candidates.
\end{lemma}

\begin{proof}
Consider the complete bipartite graph with two independent sets $V_1$ and $V_2$, where:
\begin{itemize}
\setlength\itemsep{0.1em}
    \item each vertex in $V_1$ corresponds to a candidate in $C$,
    \item each vertex in $V_2$ corresponds to a position in $\{1,\ldots,m\}$.
\end{itemize}
The edge \(\{c,p\}\) between $c\!\in\!V_1$ and $p\!\in\!V_2$ is weighted by \(
\sum_{r' \in \PP} \textstyle\sum_{i=\min\{p,\rk(c,r')\}}^{\max\{p,\rk(c,r')\}-1}N^{m-i-1}_{k-2}
\).
Each perfect matching corresponds to a ranking \(r\), where \(\rk(c,r)\!=\!p\) if edge \(\{c,p\}\) belongs to the matching. Determining a minimum weight matching in this graph thus amounts to solve the following minimization problem:
\[
\min_r \sum_{c \in C} \sum_{r' \in \PP} \textstyle\sum_{i=\min\{\rk(c,r),\rk(c,r')\}}^{\max\{\rk(c,r),\rk(c,r')\}-1}N^{m-i-1}_{k-2}.
\]
The two first sum operators can be swapped, which yields:
\[
\min_r \sum_{r' \in \PP} \sum_{c \in C} \textstyle\sum_{i=\min\{\rk(c,r),\rk(c,r')\}}^{\max\{\rk(c,r),\rk(c,r')\}-1}N^{m-i-1}_{k-2} = \min_r \sum_{r' \in \PP} \ESpearman{k}(r,r').
\]
Thus a minimum weight matching corresponds to a consensus ranking for the $k$-wise Spearman distance. Such a matching can be computed in time \(O(|V_1|^3)\) by the Hungarian algorithm, thus in \(O(m^3)\) as \(|V_1|\!=\!m\). The computation of the complete bipartite graph itself is performed in \(O(nm^3)\) (time required for computing the weights of edges). The overall complexity is therefore \(O(nm^3)\).
\end{proof}

From Lemma~\ref{lem:EDiaconisGraham}, which implies that any consensus ranking for the \(k\)-wise Spearman distance is a 2-approximation of an optimal consensus ranking for the \(k\)-wise Kemeny rule, and Lemma~\ref{lem:Spearman}, which establishes that a consensus ranking for the \(k\)-wise Spearman distance can be computed in \(O(nm^3)\), we deduce the main result of this section:
\begin{theorem}
\HG{There exists a 2-approximation algorithm for \(k\)-KAP which runs in \(O(nm^3)\) time.}
\end{theorem}

\revtwo{Note that determining whether the bound 2 is tight remains an open problem (as is the case for $k\!=\!2$, to our knowledge).}

\color{black}

\section{Numerical Tests\label{sec:Numerical}}
The numerical tests\footnote{Implementation in C++, except for the polynomial-time 2-approximation algorithm implemented in Python3. \revtwo{Unless otherwise stated,} all times are CPU seconds on an Intel Core I7-8700 3.20 GHz processor with 16GB of RAM.} we carried out had several objectives:
\begin{itemize}
    \item to evaluate the computational performance of the dynamic programming approach of Section~\ref{sec:CompAndExact},
    \item to evaluate the impact of parameter $k$ on the set of consensus rankings,
    \item to assess the efficiency of the preprocessing technique of Section~\ref{sec:Graph},
    \item to study the practical approximation ratio of the polynomial-time 2-approximation algorithm proposed in Section~\ref{sec:approx}.
\end{itemize} 

\paragraph{Generation of preference profiles}
The preference profiles are generated according to the Mallows model~\citep{mallows10.2307/2333244}, using the Python package PrefLib-Tools~\citep{MaWa13a} in most tests (the PerMallows R package~\citep{irurozki2016permallows} was used for the tests of the 2-approximation algorithm). 
This model takes two parameters as input: a reference ranking $\sigma$ (the mode of the distribution) and a dispersion parameter $\phi \in (0,1)$. Given these inputs, the probability of generating a ranking \(r\) is proportional to \(\phi^{\KendallTau(r,\sigma)}\). The more \(\phi\) tends towards 0 (resp. 1), the more the preference rankings become correlated and resemble \(\sigma\) (resp. become equally probable, i.e., we are close to the \emph{impartial culture assumption}). 
This model enables us to measure in a simple way how the level of correlation in the input rankings impacts our results. 
In all tests, the number \(n\) of voters is set to 50 and
the ranking \(\sigma\) is set arbitrarily as the \(k\)-wise Kemeny rule is neutral. 
For each triple \((m,k,\phi)\) considered, the results are averaged over 50 preference profiles.

\paragraph{Practicability of the dynamic programming approach} 
We first evaluate our dynamic programming approach on instances with different values for \(m\) and \(k\). 
Note that the computational performance measured here is not impacted by the level of correlation in the input rankings as it does not change the number of states in dynamic programming nor the computation time to determine the optimal value in each state. 
Hence, we only consider instances generated under the impartial culture assumption, i.e., with \(\phi\approx 1\). 

Table \ref{tab:NT1} (Rows 3-5) displays the average, max and min running times obtained for some representative \((m,k)\) values. As expected, the running times increase exponentially with \(m\). Conversely, parameter \(k\) seems to have a moderate impact on the running times. The dynamic programming approach enables us to solve \(k\)-KAP in a time of up to 3 sec. (resp. 76 sec.) for \(m\!\le\!14\) (resp. \(m\!\le\!18\)).

\paragraph{Influence of \(k\) on the set of consensus rankings}
Second, we study the impact of \(k\) on the set of optimal solutions to \(k\)-KAP. Indeed, one criticism for the Kemeny rule is that there exists instances for which the set of consensus rankings is compounded of many solutions which are quite different from one another. Thus, we investigate if increasing \(k\) helps in mitigating this issue.
For this purpose, we consider the same instances as before and compute the average number of consensus rankings denoted by \(|\mathcal{R}^*|_{\texttt{avg}}\). The results are displayed in the sixth row of Table~\ref{tab:NT1}. 
Interestingly, this measure decreases quickly with \(k\). For instance, when \(m\!=\!18\), \(|\mathcal{R}^*|_{\texttt{avg}}\) is divided by 5 when \(k\) increases from 2 to 9 and is below 2 when \(k\!=\!m\). The intuition is that \(\EKendallTau{k}\) becomes more fine-grained as \(k\) increases.

\begin{table*}[!h]
\caption{\label{tab:NT1}Average, max and min CPU times in seconds of the dynamic programming approach of Section \ref{sec:CompAndExact} for varying values of $m$ and $k$ (Rows 3 to 5). Average number of consensus rankings for increasing values of $m$ and $k$ (Row 6).} 
\begin{center}
\scalebox{0.84}{
\begin{tabular}{|l|lll|lll|lll|lll|}
\hline
\(m\)  &                           & 6                         &      &                           & 10                        &      &                           & 14                        &      &                           & 18                        &      \\ 
\hline 
\(k\)         & \multicolumn{1}{l|}{2}    & \multicolumn{1}{l|}{3}    & 6    & \multicolumn{1}{l|}{2}    & \multicolumn{1}{l|}{5}    & 10   & \multicolumn{1}{l|}{2}    & \multicolumn{1}{l|}{7}    & 14   & \multicolumn{1}{l|}{2}    & \multicolumn{1}{l|}{9}    & 18   \\ 
\hline
 Average time & \multicolumn{1}{l|}{$<$0.01} & \multicolumn{1}{l|}{$<$0.01} & $<$0.01 & \multicolumn{1}{l|}{0.07} & \multicolumn{1}{l|}{0.08} & 0.08 & \multicolumn{1}{l|}{2.52} & \multicolumn{1}{l|}{2.54} & 2.61 & \multicolumn{1}{l|}{70.93} & \multicolumn{1}{l|}{72.26} & 74.95 \\ \hline 
 Max time  & \multicolumn{1}{l|}{$<$0.01} & \multicolumn{1}{l|}{$<$0.01} & $<$0.01 & \multicolumn{1}{l|}{0.80}  & \multicolumn{1}{l|}{0.08} & 0.09 & \multicolumn{1}{l|}{2.64} & \multicolumn{1}{l|}{2.60} & 2.64 & \multicolumn{1}{l|}{71.57} & \multicolumn{1}{l|}{73.57} & 75.38 \\ \hline 
 Min time  & \multicolumn{1}{l|}{$<$0.01} & \multicolumn{1}{l|}{$<$0.01} & $<$0.01 & \multicolumn{1}{l|}{0.70}  & \multicolumn{1}{l|}{0.07} & 0.08 & \multicolumn{1}{l|}{2.49} & \multicolumn{1}{l|}{2.49} & 2.57 & \multicolumn{1}{l|}{70.27} & \multicolumn{1}{l|}{71.91} & 74.33 \\ \hline 
  \(|\mathcal{R}^*|_{\texttt{avg}}\) & \multicolumn{1}{l|}{3.00}   & \multicolumn{1}{l|}{1.20}  & \multicolumn{1}{l|}{1.05}  & \multicolumn{1}{l|}{3.84}   & \multicolumn{1}{l|}{1.24}   & \multicolumn{1}{l|}{1.10}  & \multicolumn{1}{l|}{5.36}   & \multicolumn{1}{l|}{2.36}   & \multicolumn{1}{l|}{1.16}  & \multicolumn{1}{l|}{19.70}   & \multicolumn{1}{l|}{4.12}   & \multicolumn{1}{l|}{1.47}  \\
 \hline 
\end{tabular}
}\end{center}
\end{table*}

\paragraph{Impact of  the \(3\)-wise majority graph} 
We now study the impact of the preprocessing method proposed in Section~\ref{sec:Graph} for \(k\!=\!3\). This preprocessing uses the \(k\)-wise majority digraph to divide \(k\)-KAP into several subproblems which can be solved separately by dynamic programming. 
Hopefully, when voters' preferences are correlated (i.e., for ``small'' \(\phi\) values), these subproblems become smaller and more numerous, making the preprocessing more efficient.
The results are shown in Table \ref{tab:NT2}, where the results obtained without preprocessing are also given in the last column. The obtained running times are highly dependent on \(\phi\). For instance, with \(m\!=\!18\), the average running time for solving \(3\)-KAP is above 1 minute if \(\phi\!=\!0.95\) while it is below 1 second if \(\phi\!\le\!0.85\). This gap is necessarily related to the preprocessing step, since \(\phi\) has no impact on the running time of the dynamic programming approach. To explain this significant speed-up, we display in Table~\ref{tab:NT3} the average size of the largest SCC of the \(3\)-wise majority digraph at the end of the preprocessing step. 
Unsurprisingly, this average size turns out to be correlated with \(\phi\): when \(\phi\!\le\! 0.5\), the size of the largest SCC is almost always 1. Hence, the preprocessing step is likely to yield directly a consensus ranking. In contrast, when \(\phi\!=\!0.95\), the average size of the largest SCC is close to \(m\), thus the impact of the preprocessing is low.

\begin{table}[hbt]
\centering
    \caption{\label{tab:NT2}Average, max and min CPU times (in seconds) for the \(3\)-wise Kemeny rule with preprocessing.}
\scalebox{0.84}{\begin{tabular}{|l|c|l|l|l|l|l||l|}
\hline
\multirow{2}*{$m$}                      & \multirow{2}*{$\phi$}       & \multirow{2}*{0.5}  & \multirow{2}*{0.8}  & \multirow{2}*{0.85} & \multirow{2}*{0.9}  & \multirow{2}*{0.95} & w/o \\
& & & & & & &preproc.\\
\hline
\multicolumn{1}{|c|}{} & Avg time & $<$0.01 & $<$0.01 & $<$0.01 & $<$0.01 & $<$0.01 & $<$0.01 \\ \cline{2-8} 
6                      & Max time  & $<$0.01 & $<$0.01 & $<$0.01 & $<$0.01 & $<$0.01 & $<$0.01\\ \cline{2-8} 
                       & Min time  & $<$0.01 & $<$0.01 & $<$0.01 & $<$0.01 & $<$0.01 & $<$0.01\\ \hline
                       & Avg time & 0.03 &  0.03 & 0.04 & 0.07 & 0.10 & 0.07   \\ \cline{2-8} 
10                     & Max time  & 0.03 &  0.03 & 0.10 & 0.15 &  0.17 &  0.08 \\ \cline{2-8} 
                       & Min time  & 0.03 &  0.03 & 0.03 & 0.03 &  0.03 & 0.07  \\ \hline
                       & Avg time & 0.09 & 0.09 & 0.11 &  0.94 &  2.21 & 2.52  \\ \cline{2-8} 
14                     & Max time  & 0.09 & 0.13 & 0.25 &  3.14  &   3.26 & 2.59  \\ \cline{2-8} 
                       & Min time  & 0.08 & 0.08 & 0.09 &  0.10  &  0.26 &  2.49  \\ \hline
                       & Avg time & 0.20 & 0.20 & 0.55 &  14.87  &  61.72 &  71.17  \\ \cline{2-8} 
18                     & Max time  & 0.31 & 0.21 & 8.46 &  79.87 & 80.11 & 71.61 \\ \cline{2-8} 
                       & Min time  & 0.19 & 0.19 & 0.19 &   0.22   &  6.02 & 71.02  \\ \hline
\end{tabular}}
\end{table}

\begin{table}[hbt]
\centering
\caption{\label{tab:NT3}Average size of the largest SCC after preprocessing.}
\scalebox{0.84}{\begin{tabular}{|c|c|c|c|c|c|}
\hline
$m$\textbackslash{}$\phi$ & 0.47            & 0.81  & 0.85 & 0.88  & 0.95  \\ \hline
6                  & \textless{}1.10 & 1.84 & 1.88 & 2.72 & 3.28  \\ \hline
10                 & \textless{}1.10 & 1.64 & 3.28 & 5.32 & 8.20   \\ \hline
14                 & \textless{}1.10 & 2.68 & 3.84 & 9.12 & 12.91  \\ \hline
18                 & \textless{}1.10 & 2.84 & 4.27 & 9.80 & 17.44 \\ \hline
\end{tabular}}
\end{table}

\revised{\paragraph{Polynomial-time 2-approximation algorithm} Lastly, we study the practical approximation ratio of the proposed polynomial-time 2-approximation algorithm. The results are averaged over 50 instances randomly generated according to the Mallows model, for 50 voters and 12 candidates. Preliminary tests not reported here tended to show that this approximation ratio does not depend on the number of candidates. Furthermore, the approximation ratio becomes better when the number of voters increases which can be explained by the fact that the law of large numbers make dominant the reference ranking in the Mallows model, regardless of the aggregation rule. We report in Figure~\ref{fig:kSpearman} the box plots obtained by varying the values of $k$ and $\phi$ (the closer to 1, the closer to the impartial culture assumption). \revtwo
{Note that the CPU time to compute a consensus ranking for the $k$-wise Spearman distance (for 50 voters and 12 candidates) is below 12 milliseconds whatever the values of $k$ and $\phi$, using the \texttt{linear\_sum\_assignment} function of the SciPy Python library on an Intel Core i5 2.3 GHz dual core processor with 8GB of RAM.} It is interesting to observe that the practical approximation ratio is much better than 2: the worst ratio over all generated instances is 1.04. For the same reason as for the number of voters, the more correlated the preferences in the profile, the better the approximation ratio. Furthermore, the greater the value of $k$, the better the approximation ratio, as is clear from the curves obtained.}

\begin{figure}[ht]
\centering
    \caption{\revised{Practical approximation ratio of the 2-approximation algorithm.}}
    \label{fig:kSpearman}
    \scalebox{0.7}{\includegraphics{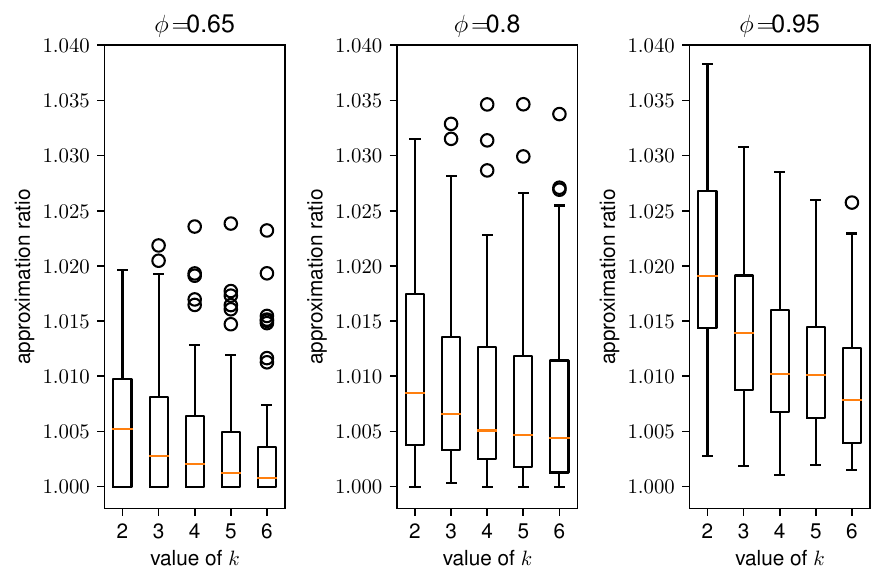}}
\end{figure}

\section{Conclusion}
In this paper, we advocate using the results of \emph{setwise} contests between candidates to design social welfare functions that are less myopic than those only based on pairwise comparisons. 
\HG{One natural such social welfare function is a \(k\)-wise generalization of the Kemeny rule, which returns a ranking minimizing the number of disagreements on top candidates of sets of cardinality lower than or equal to $k$. 
We have studied this \(k\)-wise Kemeny rule from both axiomatic and algorithmic viewpoints. In more detail, we} established that determining a consensus ranking is NP-hard for any \(k \ge 3\). Then, after proposing a dynamic programming procedure, we have investigated a \(k\)-wise variant of the majority graph, from which we developed a preprocessing step. Computing this graph is a polynomial time problem for \(k = 3\) but becomes NP-hard for \(k \ge 4\). The numerical tests show the practicability of the approach for up to 18 candidates. \HG{Lastly, we have designed a 2-approximation algorithm for the \(k\)-wise Kemeny aggregation problem. This approximation algorithm in fact returns a ranking minimizing a $k$-wise variant of the Spearman distance, and the worst case approximation ratio is then derived from an adaptation of the Diaconis-Graham inequality. The numerical tests suggest that, in practice, the \(k\)-wise Kemeny score of the returned ranking is often much better than a 2-approximation of the optimal score.}

\HG{Several research directions could be further investigated. One of them would be} to investigate the complexity of determining a consensus ranking for \(\EKendallTau{k}\) when \(k = m\), because our hardness result only holds for \emph{fixed} values of \(k\). \HG{Secondly, note that \(\EKendallTau{k}\) focuses on ``small'' sets as we count disagreements on sets of size lower than or equal to $k$. An opposite viewpoint would be to consider ``large'' sets by counting disagreements on sets of size greater than or equal to $m - k$. Note that for $k = 0$ this would lead to a rule similar to plurality. 
Another direction is to study the complexity of recognition problems related to the $k$-wise Kemeny rule~\cite{hudry2013complexity}, i.e., deciding if a given ranking is a consensus ranking for the $k$-wise Kemeny rule.} Alternative definitions of \(k\)-wise majority graphs that are easier to compute for \(k\!>\!3\) are also worth investigating. Finally, other social welfare functions based on the results of setwise contests are worth investigating in our opinion, both from the axiomatic and the computational points of view.


\section*{Acknowledgements}
\revtwo{We wish to thank the anonymous reviewers for their useful comments that helped us to improve the presentation of the paper.} This work has been partially supported by project ANR-14-CE24-0007-01 ``CoCoRICo-CoDec'' and the Italian MIUR PRIN 2017 project ALGADIMAR ``Algorithms, Games, and Digital Markets''.

\bibliographystyle{elsarticle-num-names}
\bibliography{biblio}

\end{document}